\documentclass[12pt,a4paper,final]{iopart}

\usepackage{color}
\usepackage{iopams}  
\usepackage{graphicx}
\usepackage[breaklinks=true,colorlinks=true,linkcolor=blue,urlcolor=blue,citecolor=blue]{hyperref}

\expandafter\let\csname equation*\endcsname\relax

\expandafter\let\csname endequation*\endcsname\relax

\usepackage{amsmath}
\usepackage{amsthm}
\usepackage[]{algorithm2e}
\usepackage{multirow}
\usepackage[super]{nth}
\usepackage{subcaption}
\usepackage{tabularx}

\definecolor{darkred}{rgb}{.7,0,0}

\definecolor{darkgreen}{rgb}{0,0.7,0}

\definecolor{darkblue}{rgb}{0,0,0.7}

\newcommand{\argmin}{\mathop{\mathrm{arg\,min}}}

\theoremstyle{remark}
\newtheorem{example}{Example}[section]
\newtheorem{theorem}{Proposition}

\begin{document}

\title[EKI: A Derivative-Free Technique For Machine Learning Tasks]{Ensemble Kalman Inversion: A Derivative-Free Technique For Machine Learning Tasks}

\author{Nikola B. Kovachki}
\address{Computing and Mathematical Sciences, California Institute of Technology, Pasadena, California 91125, USA}
\ead{nkovachki@caltech.edu}

\author{Andrew M. Stuart}
\address{Computing and Mathematical Sciences, California Institute of Technology, Pasadena, California 91125, USA}
\ead{astuart@caltech.edu}

\begin{abstract}
The standard probabilistic perspective on machine learning gives rise
to empirical risk-minimization tasks 
that are frequently solved by stochastic gradient descent (SGD) and variants
thereof. We present a formulation of these tasks as 
classical inverse or filtering problems and, furthermore, we propose an efficient,
gradient-free algorithm for finding a solution to these problems using
ensemble Kalman inversion (EKI). 
Applications of our approach include
offline and online supervised learning with deep neural networks, 
as well as graph-based semi-supervised learning. 
The essence of the EKI procedure is an ensemble based approximate gradient
descent in which derivatives are replaced by differences from within the
ensemble.  We suggest several modifications to the basic method, derived from
empirically successful heuristics developed in the context of SGD.  
Numerical results demonstrate  wide applicability and robustness 
of the proposed algorithm.

\end{abstract}

\noindent{\it Keywords}: Machine learning, Deep learning, Derivative-free optimization, Ensemble Kalman inversion, Ensemble Kalman filtering.
\submitto{\IP}

\section{Introduction}

\subsection{The Setting}
The field of machine learning has seen enormous advances over the last 
decade.  These advances have been driven by two key elements: (i) the 
introduction of flexible architectures which have the expressive power needed to 
efficiently represent the input-output maps encountered in practice;
(ii) the development of smart optimization tools which train the
free parameters in these input-output maps to match data. 
The text \cite{deeplearningbook} overviews the start-of-the-art.

While there is little work in the field of derivative-free, 
paralellizable methods for machine learning tasks, such advancements are 
greatly needed. Variants on the Robbins-Monro algorithm 
\cite{robbinsmonroe}, such as  stochastic gradient descent (SGD), have 
become state-of-the-art for practitioners in machine learning \cite{deeplearningbook} and an attendant theory \cite{dieuleveut2016nonparametric,bach2013non,schmidt2017minimizing,lee2016gradient,jordan2017gradient} is emerging. 
However the approach faces many challenges and limitations \cite{xavier,hardrnn1,admm}.
New directions are needed to overcome them, especially for parallelization, as attempts 
to parallelize SGD have seen limited success \cite{elasticaveraging}.

A step in the direction of a derivative-free, parallelizable algorithm for the training of 
neural networks was attempted in \cite{mac} by use of the the method of auxiliary coordinates (MAC). 
Another approach using the alternating direction method of multipliers (ADMM) and a Bregman iteration
is attempted in \cite{admm}. Both methods seem successful but are only demonstrated on supervised learning tasks with
shallow, dense neural networks that have relatively few parameters. In reinforcement learning,
genetic algorithm have seen some success (see \cite{genetic} and references therein), but it is not clear how to deploy them outside of that domain. 

To simultaneously address the issues of parallelizable and derivative-free
optimization, we demonstrate in this paper the potential for using
ensemble Kalman methods to undertake machine learning tasks.
Optimizing neural networks via Kalman filtering has been attempted before (see \cite{kalmannnbook} and references therein),
but most have been through the use of Extended or Unscented Kalman Filters. 
Such methods are plagued by 
inescapable computational and memory constraints and
hence their application has been restricted to 
small parameter models. 
A contemporaneous paper by Haber et al \cite{haber} has introduced a variant on
the ensemble Kalman filter, and applied it to the training of neural networks;
our paper works with a more standard implementations of ensemble Kalman
methods for filtering and inversion \cite{dataassim,enkfinverse} and 
demonstrates potential for these methods within a wide range of 
machine learning tasks.

\subsection{Our Contribution}

The goal of this work is two-fold:

\begin{itemize}
\item First we show that many of the common tasks considered 
in machine learning can be formulated in the unified framework of Bayesian inverse problems. 
The advantage of this point of view is that it allows for the transfer
of theory and algorithms developed for inverse problems to the field of
machine learning, in a manner accessible to the inverse problems community.
To this end we give a precise, mathematical
description of the most common approximation architecture in machine learning, 
the neural network (and its variants); we use the language of dynamical systems,
and avoid references to the neurobiological language and notation more 
common-place in the applied machine learning literature. 

\item Secondly, adopting the inverse problem point of view, we show that 
variants of ensemble Kalman methods (EKI, EnKF) can be just as effective at 
solving most machine learning tasks as the plethora of gradient-based methods 
that are widespread in the field. We borrow some ideas from 
machine learning and optimization to modify these ensemble  methods, to 
enhance their performance.
\end{itemize}

Our belief is that by formulating machine learning tasks as inverse problems,
and by demonstrating the potential for methodologies to be transferred from
the field of inverse problems to machine learning, we will open up new ways
of thinking about machine learning which may ultimately lead to 
deeper understanding of the optimization tasks at the heart of the field,
and to improved methodology for addressing those tasks. 
To substantiate the second
assertion we give examples of the competitive application of ensemble
methods to supervised, semi-supervised, and online learning problems with 
deep dense, convolutional, and recurrent neural networks. To the best of 
our knowledge, this is the first paper to successfully apply ensemble Kalman
methods to such a range of relatively large scale machine learning tasks.
Whilst we do not attempt parallelization, ensemble methods are 
easily parallelizable and we give references to relevant literature. 
Our work leaves many open questions and future research directions for
the inverse problems community.

\subsection{Notation and Overview}

We adopt the notation \(\mathbb{R}\) for the real axis, \(\mathbb{R}_+\) 
the subset of non-negative reals, and \(\mathbb{N} = \{0,1,2,\dots\}\) for the set of natural numbers. 
For any set \(A\), we use \(A^n\) to denote its \(n\)-fold Cartesian product for any \(n \in \mathbb{N} \setminus \{0\}\). 
For any function \(f : A \rightarrow B\), we use \(\text{Im}(f) = \{ y \in B : 
y=f(x),{\rm for}\,{\rm some}\, x \in A\}\) to denote its image. For any subset \(V \subseteq \mathcal{X}\) of a linear space \(\mathcal{X}\), we let
\(\dim V\) denote the dimension of the smallest subspace containing \(V\). For any Hilbert space \(\mathcal{H}\), we adopt the notation \(\|\cdot\|_\mathcal{H}\) and \(\langle \cdot, \cdot \rangle_\mathcal{H}\) to be its associated norm and inner-product respectively. Furthermore for any symmetric, positive-definite operator \(C: \mathcal{D}(C) \subset \mathcal{H} \rightarrow \mathcal{H}\), we use the notation \(\|\cdot\|_C = \|C^{-\frac{1}{2}} \cdot\|_\mathcal{H}\) and \(\langle \cdot,  \cdot \rangle_C = \langle C^{-\frac{1}{2}} \cdot, C^{-\frac{1}{2}} \cdot, \rangle_\mathcal{H}\). For any two topological spaces \(\mathcal{X}, \mathcal{Y}\), 
we let \(C(\mathcal{X},\mathcal{Y})\) denote the set of continuous functions from \(\mathcal{X}\) to \(\mathcal{Y}\).
We define
\[\mathbb{P}^m =  \{y \in \mathbb{R}^m \; | \; \|y\|_1 = 1, y_1,\dots,y_m \geq 0\}\]
the set of \(m\)-dimensional probability vectors, and
the subset
\[\mathbb{P}^m_0 = \{y \in \mathbb{R}^m \; | \; \|y\|_1 = 1, y_1,\dots,y_m > 0\}.\]

Section \ref{learningproblem} delineates the learning problem, starting from the 
classical, optimization-based framework, and shows how it can be formulated as a 
Bayesian inverse problem. Section \ref{approxarchitectures} gives a brief overview of
modern neural network architectures as dynamical systems. Section \ref{algorithms} outlines 
the state-of-the-art algorithms for fitting neural network models, as well as the EKI method and 
our proposed modifications of it. 
Section \ref{numericalexperiments} presents our numerical experiments,
comparing and contrasting EKI methods with the state-of-the-art. Section \ref{conclusion} gives 
some concluding remarks and possible future directions for this line of work.

\section{Problem Formulation}
\label{learningproblem}

Subsection \ref{classicalframework} overviews the standard formulation of machine learning 
problems with subsections \ref{supervisedlearning}, \ref{semisupervisedlearning}, and
\ref{onlinelearning}  presenting supervised, semi-supervised, and online learning respectively.
Subsection \ref{inverseproblems} sets forth the Bayesian inverse problem interpretation of 
these tasks and gives examples for each of the previously presented problems.

\subsection{Classical Framework}
\label{classicalframework}

The problem of learning is usually formulated as minimizing an expected cost over some space of mappings relating the data \cite{deeplearningbook, statisticallearning,murphyml}.
More precisely, let \(\mathcal{X}\), \(\mathcal{Y}\) be separable Hilbert spaces and let \(\mathbb{P} (x,y)\) be a probability measure on the product space \(\mathcal{X} \times \mathcal{Y}\). Let \(\mathcal{L}: \mathcal{Y} \times \mathcal{Y} \rightarrow \mathbb{R}_+\) be a positive-definite function and define 
\(\mathcal{F}\) to be the set of mappings \(\{\mathcal{G}: \mathcal{X} \rightarrow \mathcal{Y}\}\) on which the composition \(\mathcal{L}(\mathcal{G}(\cdot),\cdot)\) is \(\mathbb{P}\)-measurable for all \(\mathcal{G}\) in \(\mathcal{F}\). 
Then we seek to minimize the functional
\begin{equation}
\label{eq:min}
Q(\mathcal{G}) = \int_{\mathcal{X} \times \mathcal{Y}} \mathcal{L}(\mathcal{G}(x),y) \; d \mathbb{P} (x,y).
\end{equation}
across all mappings in \(\mathcal{F}\). This minimization may not be well defined as there could be 
infimizing sequences not converging in \(\mathcal{F}\). Thus further constraints (regularization) are needed to obtain an 
unambiguous optimization problem. These are generally introduced by working 
with parametric forms of \(\mathcal{G}\). Additional, explicit regularization
is also often added to parameterized versions of \eqref{eq:min}.

Usually \(\mathcal{L}\) is called the loss or cost function and acts as a metric-like function on \(\mathcal{Y}\); however it is useful in applications to
relax the strict properties of a metric, and we, in particular, do not
require \(\mathcal{L}\) to be symmetric or subadditive. With this interpretation of
\(\mathcal{L}\) as a cost, we are seeking a
mapping \(\mathcal{G}\) with lowest cost, on average with respect to
\(\mathbb{P}\). There are numerous choices for \(\mathcal{L}\) used in applications \cite{deeplearningbook}; some of the most common include the squared-error loss \(\mathcal{L}(y',y) = \|y - y'\|^2_\mathcal{Y}\) used for regression tasks, and the cross-entropy loss \(\mathcal{L}(y',y) = - \langle y, \log y' \rangle_\mathcal{Y}\) used for classification tasks. In both these cases we often have \(\mathcal{Y}=\mathbb{R}^K\), and, for
classification, we may restrict the class of mappings to those taking values
in \(\mathbb{P}^K\).

Most of our focus will be on \textit{parametric} learning where we approximate \(\mathcal{F}\) by a parametric family of models \(\{\mathcal{G}(u|\cdot): \mathcal{X} \rightarrow \mathcal{Y}\}\) where \(u \in \mathcal{U}\) is the parameter and \(\mathcal{U}\) is a separable Hilbert space. This allows us to work with a computable class of functions and perform the minimization directly over \(\mathcal{U}\). Much of the early work in machine learning focuses on model classes which make the associated minimization problem convex \cite{firstsvm, ridgeregression, murphyml}, but the recent empirical success of neural networks has driven research away from this direction \cite{deeplearningnature,deeplearningbook}. In Section \ref{approxarchitectures}, we give a brief overview of the model space of 
neural networks.

While the formulation presented in \eqref{eq:min}
is very general, it is not directly transferable
to practical applications as, typically, we have no direct access to 
\(\mathbb{P}(x,y)\). How we choose to address this issue depends on the information known to us, usually in the form of a data set, 
and defines the type of learning. Typically information about \(\mathbb{P}\) is accessible only through our
sample data. The next three subsections describe particular structures
of such sample data sets which arise in applications, and the minimization
tasks which are constructed from them to determine the parameter \(u\).

\subsubsection{Supervised Learning}
\label{supervisedlearning}

Suppose that we have a dataset \(\{(x_j,y_j)\}_{j=1}^N\) assumed to be i.i.d. samples from \(\mathbb{P}(x,y)\). We can thus replace the integral
\eqref{eq:min}  with its Monte Carlo approximation, and add a regularization
term, to obtain the following  minimization problem: 
\begin{align}
\label{supervised_classic}
&\argmin_{u \in \mathcal{U}} \Phi_{\text{s}}(u;\mathsf{x},\mathsf{y}), \\
& \Phi_{\text{s}}(u;\mathsf{x},\mathsf{y}) = \frac{1}{N} \sum_{j=1}^N \mathcal{L}(\mathcal{G}(u|x_j),y_j) + R(u).
\end{align}
Here \(R: \mathcal{U} \rightarrow \mathbb{R}\) is a regularizer on the 
parameters designed to prevent overfitting or address possible ill-posedness. 
We use the notation \(\mathsf{x} = [x_1,\dots,x_N] \in \mathcal{X}^N\), and
analogously \(\mathsf{y}\),  for concatenation of the data in the input
and output spaces \(\mathcal{X}, \mathcal{Y}\) respectively.

A common choice of regularizer is
\(R(u) = \lambda \|u\|_\mathcal{U}^2\) where \(\lambda \in \mathbb{R}\) is a tunable parameter. This choice
is often called \textit{weight decay} in the machine learning literature. Other choices, such as sparsity promoting norms, are
also employed; carefully selected choices of the norm can induce desired behavior in the parameters \cite{l1min, computationalinverseproblems}. 
We note also that Monte Carlo approximation is itself a form of regularization
of the minimization task \eqref{eq:min}.

This formulation 
is known as \textit{supervised} learning. Supervised learning is perhaps the most common type of machine learning with numerous applications 
including image/video classification, object detection, and natural language processing \cite{alexnet,nlpfoundation,seqtoseq}.

\subsubsection{Semi-Supervised Learning}
\label{semisupervisedlearning}

Suppose now that we only observe a small portion of the data \(\mathsf{y}\) in the image
space; specifically we assume that we have access to data
\(\{x_j\}_{j \in Z}\), \(\{y_j\}_{j \in Z'}\)
where \(x_j \in \mathcal{X}, y_j \in \mathcal{Y}\), 
\(Z = \{1,\dots,N\}\) and where \(Z' \subset Z\) with \(|Z'| \ll |Z|\). Clearly this can be turned into supervised 
learning by ignoring all data indexed by \(Z \setminus Z'\), but we would like to take advantage of all the information known to us.
Often the data in \(\mathcal{X}\) is known as unlabeled data,
and the data in \(\mathcal{Y}\) as labeled data; in particular
the labeled data is often in the form of categories. We use
the terms labeled and unlabeled in general, regardless or whether
the data in \(\mathcal{Y}\) is categorical; however some of our illustrative
discussion below will focus on the binary classification problem.
The objective is to assign a label
\(y_j\) to every \(j \in Z\). This problem is known as \textit{semi-supervised} 
learning.

One approach to the problem is to seek to minimize 
\begin{align}
\label{semi_supervised_classic}
&\argmin_{u \in \mathcal{U}} \Phi_{\text{ss}}(u;\mathsf{x},\mathsf{y}) \\ 
&\Phi_{\text{ss}}(u;\mathsf{x},\mathsf{y}) = \frac{1}{|Z'|} \sum_{j \in Z'} \mathcal{L}(\mathcal{G}(u|x_j),y_j) + R(u;\mathsf{x})
\end{align}
where the regularizer \(R(u;\mathsf{x})\) 
may use the unlabeled data in \(Z\backslash Z'\), 
but the loss term involves only labeled data in \(Z'\).

There are a variety of ways in which one can construct the regularizer \(R(u;\mathsf{x})\) including graph-based and low-density separation methods \cite{semisupervised,andreasemi}. In this work, we will study a nonparametric graph approach where we think of \(Z\) as indexing the nodes on a graph. To illustrate ideas we consider the case
of binary outputs, take \(\mathcal{Y} = \mathbb{R}\) and restrict attention
to mappings \(\mathcal{G}(u|\cdot)\) which take values in \(\{-1,1\}\);
we sometimes abuse notation and simply take \(\mathcal{Y} = \{-1,1\}\), so that \(\mathcal{Y}\)
is no longer a Hilbert space.
We assume that \(\mathcal{U}\) comprises real-valued functions on the
nodes \(Z\) of the graph, equivalently vectors in \(\mathbb{R}^N\). 
We specify
that \(\mathcal{G}(u|j) = \text{sgn}(u(j))\) for all \(j \in Z\),
and take, for example, the probit or logistic loss 
function \cite{RW,semisupervised}.
Once we have found an optimal parameter value for
\(u: Z \rightarrow \mathbb{R}\), 
application of \(\mathcal{G}\) to \(u\) will return a labeling over all nodes
\(j\) in \(Z\). In order to use all the unlabeled data we introduce edge
weights which measure affinities between nodes of a graph with vertices \(Z\),
by means of a weight function on \(\mathcal{X} \times \mathcal{X}\). We then compute the graph Laplacian \(\mathsf{L}(\mathsf{x})\) and use it to define 
a regularizer in the form
\[R(u;\mathsf{x}) = \langle u, (\mathsf{L}(\mathsf{x}) + \tau^2 I)^{\alpha} u \rangle_{\mathbb{R}^N}.\]
Here \(I\) is the identity operator, and \(\tau, \alpha \in \mathbb{R}\) with \(\alpha > 0\) are tunable parameters. Further details of this method are in the following section. Applications of semi-supervised learning can include any situation where data in the image space \(\mathcal {Y}\) is hard to come by, for example because it requires expert human labeling; a specific example is medical imaging \cite{medicalimaging}.

\subsubsection{Online Learning}
\label{onlinelearning}

Our third and final class of learning problems concerns situations where samples of data are presented to us sequentially and we aim to refine our choice of parameters at each step. 
We thus have the supervised learning problem \eqref{supervised_classic} and we aim to
solve it sequentially as each pair of data points \(\{x_j,y_j\}\) is
delivered.
To facilitate cheap algorithms we impose a Markovian structure
in which we are allowed to use only the current data sample, as well as our 
previous estimate of the parameters, when searching for the new estimate. 
We look for a sequence
\(\{u_j\}_{j=1}^\infty \subset \mathcal{U}\) such that \(u_j \rightarrow u^*\) as \(j \rightarrow \infty\) where, in the perfect scenario, \(u^*\) will
be a minimizer of the limiting learning problem \eqref{eq:min}.
To make the problem Markovian, we may formulate it as the following minimization task
\begin{align}
\label{online_classic}
&u_j = \argmin_{u \in \mathcal{U}} \Phi_{\text{o}}(u,u_{j-1};x_j, y_j) \\
&\Phi_{\text{o}}(u,u_{j-1};x_j, y_j) = \mathcal{L}(\mathcal{G}(u|x_j),y_j) + R(u;u_{j-1})
\end{align}
where \(R\) is again a regularizer that could enforce a closeness condition between consecutive parameter estimates, such as 
\[R(u;u_{j-1}) = \lambda  \|u - u_{j-1}\|^2_\mathcal{U}.\] 
Furthermore this regularization need not be this explicit, but could rather be included in the method chosen to solve \eqref{online_classic}. For example if we use an iterative method for the minimization, we could simply start the iteration at \(u_{j-1}\). 

This formulation of supervised learning is known as \textit{online} learning. It can be viewed as reducing computational cost as a cheaper, sequential way of estimating a solution to \eqref{eq:min}; or it may be necessitated by the 
sequential manner in which data is acquired.  

\subsection{Inverse Problems}
\label{inverseproblems}

The preceding discussion demonstrates that, 
while the goal of learning is to find a mapping 
which generalizes across the whole distribution of 
possible data, in practice, we are severely restricted by only having access to a finite data set. Namely formulations \eqref{supervised_classic}, \eqref{semi_supervised_classic}, \eqref{online_classic} can be stated for any
input-output pair data set with no reference to \(\mathbb{P}(x,y)\) by simply assuming that there exists some function in our model class that will relate the two.
In fact, since \(\mathcal{L}\) is positive-definite, its dependence also washes out when ones takes a function approximation point of view. To make this precise, consider the inverse problem of finding \(u \in \mathcal{U}\) such that
\begin{equation}
\label{static_inverse}
\mathsf{y} = \mathsf{G}(u|\mathsf{x}) + \eta;
\end{equation}
here \(\mathsf{G}(u|\mathsf{x}) = [\mathcal{G}(u|x_1),\dots,\mathcal{G}(u|x_N)]\) is a concatenation and \(\eta \sim \pi\) is a \(\mathcal{Y}^N\)-valued random variable distributed according to a measure \(\pi\) that models possible noise in the data, or model error. In order to facilitate a Bayesian formulation of this
inverse problem  we let \(\mu_0\) denote a prior probability measure on the parameters \(u\). Then supposing
\begin{align*}
- \log (\pi (\mathsf{y} - \mathsf{G}(u|\mathsf{x}))) &\propto \sum_{j=1}^N \mathcal{L}(\mathcal{G}(u|x_j),y_j) \\
- \log (\mu_0(u)) &\propto R(u)
\end{align*}
we see that \eqref{supervised_classic} corresponds to the standard MAP 
estimator arising from a Bayesian formulation of \eqref{static_inverse}. The semi-supervised learning problem \eqref{semi_supervised_classic} can also be viewed as a MAP estimator by restricting \eqref{static_inverse} to \(Z'\) and using \(\mathsf{x}\) to build \(\mu_0\). This is the perspective we take in this work and we illustrate with an example for each type of problem. \\

\begin{example}
Suppose that \(\mathcal{Y}\) and \(\mathcal{U}\) are Euclidean spaces and let \(\pi = \mathcal{N}(0,\Gamma)\) and \(\mu_0 = \mathcal{N}(0, \Sigma)\) be Gaussian with positive-definite covariances \(\Gamma, \Sigma\) where \(\Gamma\) is block-diagonal with \(N\) identical blocks \(\Gamma_0\). Computing the MAP estimator of \eqref{static_inverse}, we obtain that
\(\mathcal{L}(y',y) = \|y - y'\|_{\Gamma_0}^2\) and \(R(u) = \|u\|_\Sigma^2\). 
\end{example}
\*
\begin{example}
\label{ex:2.2}
Suppose that \(\mathcal{U} = \mathbb{R}^N\) and \(\mathcal{Y} = \mathbb{R}\) with the data \(y_j = \pm 1\) \(\forall j \in Z'\). We will take the model class to be a single function \(\mathcal{G}: \mathbb{R}^N \times Z \rightarrow \mathbb{R}\) depending only on the index of each data point and defined by \(\mathcal{G}(u|j) = \text{sgn}(u_j)\). As mentioned, we think of \(Z\) as the nodes on a graph and construct the edge set \(E = (e_{ij}) = \eta(x_i, x_j)\) where \(\eta: \mathcal{X} \times \mathcal{X} \rightarrow \mathbb{R}_+\) is a symmetric function. This allows construction of the associated graph Laplacian \(\mathsf{L}(\mathsf{x})\). We shift it and remove its null space and consider the symmetric, positive-definite operator \(C= (\mathsf{L}(\mathsf{x}) + \tau^2I)^{-\alpha}\) from which we can define the Gaussian measure \(\mu_0 = \mathcal{N}(0,C)\). For details on why this construction defines a reasonable prior we refer to \cite{semisupervised}. 
Letting \(\pi = \mathcal{N}(0, \frac{1}{\gamma^2} I)\), we restrict \eqref{static_inverse} to the inverse problem
\[y_j = \mathcal{G}(u|j) + \eta_j \quad \forall j \in Z'.\]
With the given definitions, letting \(\gamma^2 = |Z'|\), the associated MAP estimator has the form of \eqref{semi_supervised_classic}, namely
\[\frac{1}{|Z'|} \sum_{j \in Z'} |\mathcal{G}(u|j) - y_j|^2 + \langle u, C^{-1} u \rangle_{\mathbb{R}^N}.\]
The infimum for this functional is not achieved \cite{geometriclevelset}, but the
ensemble based methods we employ to solve the problem implicitly
apply a further regularization which circumvents this issue.
\end{example}
\*
\begin{example}
Lastly we turn to the online learning problem (\ref{online_classic}). We assume that there is some unobserved, fixed in time parameter of our model that will perfectly match the observed data up to a noise term. Our goal is to estimate this parameter sequentially. Namely, we consider the stochastic dynamical system,
\begin{align}
\label{dynamic_inverse}
\begin{split}
u_{j+1} &= u_j \\
y_{j+1} &= \mathcal{G}(u_{j+1}|x_{j+1}) + \eta_{j+1} 
\end{split}
\end{align}
where the sequence \(\{\eta_j\}\) are \(\mathcal{Y}\)-valued i.i.d. random variables that are also independent from the data. This is an instance of the classic filtering problem considered in data assimilation \cite{dataassim}. We may view this as solving an inverse problem at each fixed time with increasingly 
strong prior information as time unrolls. 
With the appropriate assumptions on the prior and the noise model, we may again view \eqref{online_classic} as the MAP estimators of each fixed inverse problem. Thus we 
may consider all problems presented here in the general framework of (\ref{static_inverse}).
\end{example}

\section{Approximation Architectures}
\label{approxarchitectures}

In this section, we outline the approximation architectures that 
we will use to solve the three machine learning tasks outlined in the
preceding section. For supervised and online learning these amount
to specifying the dependence of \(\mathcal{G}\) on \(u\); 
for semi-supervised
learning this corresponds to determining a basis in which to seek
the parameter \(u\). We do not give further details for the semi-supervised case
as our numerics fit in the context of Example \ref{ex:2.2}, 
but we refer the reader to \cite{semisupervised} for a detailed discussion.

Subsection \ref{feedforwardneuralnetworks} details feed-forward neural 
networks with subsections \ref{densenetworks} and \ref{convolutionalnetworks} showing
the parameterizations of dense and convolutional networks respectively.
Subsection \ref{recurrentneuralnetworks} presents basic recurrent neural networks.

\subsection{Feed-Forward Neural Networks}
\label{feedforwardneuralnetworks}

Feed-forward neural networks are a parametric model class defined as discrete time, nonautonomous, semi-dynamical systems of an unusual type. Each map in the composition takes a specific parametrization and can change the dimension of its input while the whole system is computed only up to a fixed time horizon. To make this precise, we will assume \(\mathcal{X} = \mathbb{R}^d\), \(\mathcal{Y} = \mathbb{R}^m\) and define a neural network with \(n \in \mathbb{N}\) hidden layers as the composition
\[\mathcal{G}(u|x) = S \circ A \circ F_{n-1} \circ \dots \circ F_0 \circ x \] 
where \(d_0 = d\) and \(F_j \in C(\mathbb{R}^{d_j}, \mathbb{R}^{d_{j+1}}), n=0, \dots, n-1\) are nonlinear maps, referred to as \textit{layers}, depending on parameters \(\theta_0, \dots, \theta_{n-1}\) respectively, \(A: \mathbb{R}^{d_n} \rightarrow \mathbb{R}^m\) is an affine map with parameters \(\theta_n\), and \(u = [\theta_0, \dots, \theta_n]\) is a concatenation.  The map \(S: \mathbb{R}^{m} \rightarrow V \subseteq \mathbb{R}^m\) is fixed and thought of as a projection or thresholding done to move the output to the appropriate subset of data space. The choice of \(S\) is dependent on the problem at hand. If we are considering a regression task and \(V = \mathbb{R}^m\) then \(S\) can simply be taken as the identity. On the other hand, if we are considering a classification task and \(V = \mathbb{P}^m\), the set of probability vectors in \(\mathbb{R}^m\), 
then \(S\) is often taken to be the softmax function defined as 
\[S(w) = \frac{1}{\sum_{j=1}^m e^{w_j}} (e^{w_1},\dots,e^{w_m}).\] 
From this perspective, the neural network approximates a categorical distribution of the input data and the softmax arises naturally 
as the canonical response function of the categorical distribution (when viewed as belonging to the exponential family of distributions) \cite{GLM, categoricaldist}. If we have some specific bounds for the output data, for example \(V = [-1,1]^m\) then \(S\) can be a point-wise hyperbolic tangent. 

What makes this dynamic unusual is the fact that each map can change the dimension of its input unlike a standard dynamical system which operates on a fixed metric space. However, note that the sequence of dimension changes \(d_1,\dots,d_n\) is simply a modeling choice that we may alter. Thus let \(d_{\max} = \max \{d_0,\dots,d_n\}\) and consider the solution map \(\phi: \mathbb{N}_0 \times \mathbb{N}_0 \times \mathbb{R}^{d_{\max}} \rightarrow \mathbb{R}^{d_{\max}}\) generated 
by the nonautonomous difference equation
\[z_{k+1} = F_k(z_k)\]
where each map \(F_k \in C(\mathbb{R}^{d_{\max}}, \mathbb{R}^{d_{\max}})\) is such that \(\dim \text{Im}(F_k) \leq d_{k+1}\); then
\(\phi(n,m,x)\) is \(z_n\) given that \(z_m=x\). We may then define a neural network as
\[\mathcal{G}(u|x) = S \circ A \circ \phi (n, 0, \mathcal{P}x)\]
where \(\mathcal{P} : \mathbb{R}^{d} \rightarrow \mathbb{R}^{d_{\max}}\) is a projection operator and \(A: \mathbb{R}^{d_\text{max}} \rightarrow \mathbb{R}^m\) 
is again an affine map. While this definition is mathematically satisfying and 
potentially useful for theoretical
analysis as there is a vast literature on nonautonomus semi-dynamical systems \cite{nonautonomousdynamics}, in practice, it is more useful to think of each map as changing the dimension of its input.
This is because it allows us to work with parameterizations that explicitly enforce the constraint on the dimension of the image. We take this point of view for the rest of this section to illustrate the practical uses of neural networks.

\subsubsection{Dense Networks}
\label{densenetworks}

A key feature of neural networks is the specific parametrization of each map \(F_k\). In the most 
basic case, each \(F_k\) is an affine map followed by a point-wise nonlinearity, in particular,
\[F_k(z_k) = \sigma(W_k z_k + b_k)\]
where \(W_k \in \mathbb{R}^{d_{k+1} \times d_k}\), \(b_k \in \mathbb{R}^{d_{k+1}}\) are the parameters i.e. \(\theta_k = [W_k, b_k]\) and \(\sigma \in C(\mathbb{R}, \mathbb{R})\) is non-constant, bounded, and monotonically increasing; we extend \(\sigma\) to a function
on \(\mathbb{R}^d\) by defining it point-wise as \(\sigma(u)_j=\sigma(u_j)\) for
any vector \(u \in \mathbb{R}^d\). This layer type is referred to as \textit{dense}, or fully-connected, because each entry in \(W_k\) is a parameter with no global sparsity assumptions and hence we can end up with a dense matrix. A neural network with only this type of layer is called dense or fully-connected (DNN).

The nonlinearity \(\sigma\), called the \textit{activation function}, is a design choice and usually does not vary from layer to layer. Some popular choices include the sigmoid, the hyperbolic tangent, or the rectified linear unit (ReLU) defined by \(\sigma(q) = \max \{0,q\}\). Note that ReLU is unbounded and hence does not satisfy the assumptions for the classical universal approximation theorem \cite{universalapproximation}, but it has shown tremendous numerical success when the associated inverse problem is solved via backpropagation (method of adjoints) \cite{originalrelu}.

\subsubsection{Convolutional Networks}
\label{convolutionalnetworks}

Instead of seeking the full representation of a linear operator at each time step, we may consider looking only for the parameters associated to a pre-specified type of operator.
Namely we consider
\[F_k(z_k) = \sigma (W(s_k)z_k + b_k)\]
where \(W\) can be fully specified by the parameter \(s_k\). The most commonly considered operator is the one arising from a discrete convolution \cite{originalconvolution}. We consider the input
\(z_k\) as a function on the integers with period \(d_k\) then we may define \(W(s_k)\) as the circulant matrix arising as the kernel of the discrete circular convolution
with convolution operator \(s_k\). Exact construction of the operator \(W\) is a modeling choice as one can pick exactly which blocks of \(z_k\) to compute the convolution over. Usually, even with maximally overlapping blocks, the operation is dimension reducing, but can be made dimension preserving, or even expanding, by appending zero entries to \(z_k\). This is called \textit{padding}. For brevity, we omit exact descriptions of such details and refer the reader to \cite{deeplearningbook}. The parameter \(s_k\) is known as the \textit{stencil}. Neural networks following this construction are called \textit{convolutional} (CNN).

In practice, a CNN computes a linear combination of many convolutions at each time step, namely 
\[F_k^{(j)}(z_k) = \sigma \left ( \sum_{m=1}^{M_k} W(s_k^{(j, m)})z_k^{(m)} + b_k^{(j)} \right )\]
for \(j=1,2,\dots,M_{k+1}\) where \(z_k = [z_k^{(1)}, \dots, z_k^{(M_k)}]\) with each entry known as a \textit{channel} and \(M_k = 1\) if no convolutions were computed at the previous iteration. Finally we define \(F_k(z_k) = [F_k^{(1)}(z_k), \dots, F_k^{(M_{k+1})}(z_k)]\). The number of channels at each time step, the
integer \(M_{k+1}\), is a design choice which, along with the choice for the size of the stencils \(s_k^{(j,m)}\), the dimension of the input, and the design of \(W\) determine the dimension of the image space \(d_{k+1}\) for the map \(F_k\).  

When employing convolutions, it is standard practice to sometimes place maps which compute certain statistics from the convolution. These operations are commonly referred to as \textit{pooling} \cite{maxpooling}. Perhaps the most common such operation is known as max-pooling. To illustrate suppose \([F^{(1)}_k, \dots, F^{(M_{k+1})}_k]\) are the \(M_{k+1}\) channels computed as the output of a convolution (dropping the \(z_k\) dependence for notational convenience). In this
context, it is helpful to change perspective slightly and view each 
\(F^{(j)}_k\) as a matrix whose concatenation gives the vector
\(F_k\). Each of these matrices is a two-dimensional grid whose value at 
each point represents a linear combination of convolutions each computed 
at that spatial 
location. We define a maximum-based, block-subsampling operation
\[(p^{(j)}_k)_{il} = \max_{q \in \{1,\dots,H_1\}} \max_{v \in \{1,\dots,H_2\}} (F^{(j)}_k)_{\alpha (i-1) + q, \beta (l-1) + v}\]
where the tuple \((H_1, H_2) \in \mathbb{N}^2\) is called the \textit{pooling kernel} and the tuple \((\alpha, \beta) \in \mathbb{N}^2\) is called the \textit{stride}, each a design choice for the operation. It is common practice to take \(H_1 = H_2 = \alpha = \beta\). We then define the full layer as \(F_k(z_k) = [p^{(1)}_k(z_k), \dots, p^{(M_{k+1})}_k(z_k)]\). There are other standard choices for pooling operations including average pooling, \(\ell_p\)-pooling, fractional max-pooling, and adaptive max-pooling 
where each of the respective names are suggestive of the operation being performed; details may be found in \cite{adaptivepooling, fractionalpooling, lppooling, averagepooling}. Note that pooling operations are dimension reducing and are usually thought of as a way of extracting the most important information from a convolution. When one chooses the kernel \((H_1,H_2)\) such that \(F^{(j)}_k \in \mathbb{R}^{H_1 \times H_2}\), the per channel output of the pooling is a scalar and the operation is called \textit{global pooling}. 

\begin{figure}[t]
\centering
\begin{tabular}{ |l|l|l| }
\hline
\multicolumn{3}{ |c| }{Convolutional Neural Network} \\
\hline
\textbf{Map} & \textbf{Type} & \textbf{Notation} \\ \hline
\multirow{2}{*}{\(F_0: \mathbb{R}^{28 \times 28} \rightarrow \mathbb{R}^{32 \times 24 \times 24}\)} & Conv 32x5x5  & \(M_0=1, M_1=32, s_0^{(j,m)} \in \mathbb{R}^{5 \times 5}\) \\ 
& & \(j \in \{1,\dots,32\}, m=\{1\}\) \\ \hline 
\multirow{3}{*}{\(F_1: \mathbb{R}^{32 \times 24 \times 24} \rightarrow \mathbb{R}^{32 \times 10 \times 10}\)} & Conv 32x5x5 & \(M_2=32, s_1^{(j,m)} \in \mathbb{R}^{5 \times 5}\) \\
 & MaxPool 2x2 & \(j \in \{1,\dots,32\}, m \in \{1,\dots,32\}\) \\
 &  & \(H_1 = H_2 = 2\) (\(\alpha = \beta = 2\)) \\ \hline
\multirow{2}{*}{\(F_2: : \mathbb{R}^{32 \times 10 \times 10} \rightarrow \mathbb{R}^{64 \times 6 \times 6}\)} & Conv 64x5x5 & \(M_3=64, s_2^{(j,m)} \in \mathbb{R}^{5 \times 5}\) \\
 & & \(j \in \{1,\dots,64\}, m \in \{1,\dots,32\}\) \\ \hline
 \multirow{3}{*}{\(F_3: \mathbb{R}^{64 \times 6 \times 6} \rightarrow \mathbb{R}^{64}\)} & Conv 64x5x5 & \(M_4=64, s_3^{(j,m)} \in \mathbb{R}^{5 \times 5}\) \\
 & MaxPool 2x2 & \(j \in \{1,\dots,64\}, m \in \{1,\dots,64\}\) \\
 & (global) & \(H_1 = H_2 = 2\) \\ \hline
\(F_4: \mathbb{R}^{64} \rightarrow \mathbb{R}^{500}\) & FC-500 & \(W_4 \in \mathbb{R}^{500 \times 64}, b_4 \in \mathbb{R}^{500}\) \\ \hline
\(A: \mathbb{R}^{500} \rightarrow \mathbb{R}^{10}\) & FC-10 & \(W_5 \in \mathbb{R}^{10 \times 500}, b_5 \in \mathbb{R}^{10}\) \\ \hline
\(S: \mathbb{R}^{10} \rightarrow \mathbb{R}^{10}\) & Softmax & \(S(w) = \frac{1}{\sum_{j=1}^{10} e^{w_j}} (e^{w_1},\dots,e^{w_{10}})\) \\ 
\hline
\end{tabular}
\caption{A four layer convolutional neural network for classifying images in the MNIST data set. The middle column shows a description typical of the 
machine learning literature. The other two columns connect this jargon to the notation presented here. No padding is added and the convolutions are computed over maximally overlapping blocks (stride of one). The nonlinearity \(\sigma\) is the ReLU and is the same for every layer.}
\label{ConvDemoDesign}
\end{figure}

\begin{figure}[t]
    \centering
    \includegraphics[width=\textwidth]{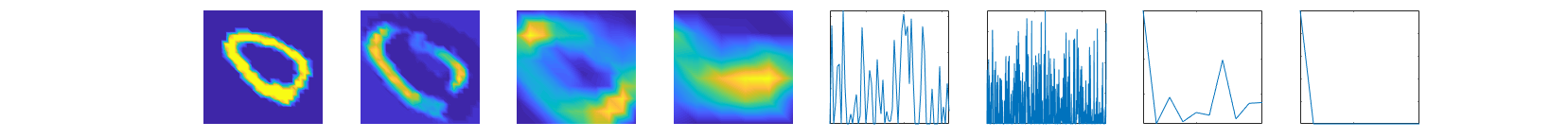}
    \caption{Output of each map from left to right of the convolutinal neural network shown in Figure \ref{ConvDemoDesign}. The left most image is the
    input and the next three images show a single randomly selected channel from the outputs of \(F_0,F_1,F_2\) respectively. The outputs of \(F_3,F_4,A,S\) are 
    vectors shown respectively in the four subsequent plots. We see that with high probability the network determines that the image belongs to the first class (0) which is 
correct.}
    \label{ConvDemoOutput}
\end{figure}

Designs of feed-forward neural networks usually employ both convolutional (with and without pooling) and dense layers. While the success of convolutional networks has mostly come from image classification or object detection tasks \cite{alexnet}, they  can be useful for any data with spatial correlations \cite{convtimeseries,deeplearningbook}. To connect the complex notation presented in this section with the standard in machine learning literature, we will give an example of a deep convolutional neural network. We consider the task of classifying images of hand-written digits given in the MNIST dataset \cite{mnist}. These are \(28 \times 28\) grayscale images of which there are \(N=60,000\) and 10 overall classes \(\{0,\dots,9\}\) hence we consider \(\mathcal{X} = \mathbb{R}^{28 \times 28} \cong \mathbb{R}^{784}\) and \(\mathcal{Y}\) the space of probability vectors over \(\mathbb{R}^{10}\). Figure \ref{ConvDemoDesign} show a typical construction of a deep convolutional neural network for this task. The word deep is generally reserved for models with \(n > 3\). Once the 
model has been fit, Figure \ref{ConvDemoOutput} shows the output of each map on an example image. Starting with the digitized digit \(0\),
the model computes its important features, through a sequence of operations
involving convolutional layers,
culminating in the second to last plot, the output of the affine map \(A\). 
This plot shows model determining that the most
likely digit is \(0\), but also giving substantial probability weight on 
the digit \(6\). This makes sense, as the digits \(0\) and \(6\) can look quite similar, especially when hand-written.
Once the softmax is taken (because it exponentiates), the probability of the image being a \(6\) is essentially washed out, as shown
in the last plot. This is a short-coming of the softmax as it may
not accurately retain the confidence of the model's prediction. We stipulate 
that this may be a reason for the emergence of highly-confident adversarial examples \cite{adversarial_original,adversarial}, but do not pursue suitable modifications
in this work.

\subsection{Recurrent Neural Networks}
\label{recurrentneuralnetworks}

Recurrent neural networks are models for time-series data defined as discrete time, nonautonomous, semi-dynamical systems 
that are parametrized by feed-forward neural networks. To make this precise, we first define a layer of two-inputs
simply as the sum of two affine maps followed by a point-wise nonlinearity, namely for \(j=1,\dots,n\) define 
\(F_{\theta_j}: \mathbb{R}^{d_h} \times \mathbb{R}^d \rightarrow \mathbb{R}^{d_h}\) by
\[
F_{\theta_j}(z, q) = \sigma(W^{(j)}_h z + b^{(j)}_h + W^{(j)}_x q + b^{(j)}_x)
\]
where \(W^{(j)}_h \in \mathbb{R}^{d_h \times d_h}\), \(W^{(j)}_x \in \mathbb{R}^{d_h \times d}\) and \(b^{(j)}_h, b^{(j)}_x \in \mathbb{R}^{d_h}\);
the parameters are then given by the concatenation
\(\theta_j = [W^{(j)}_h, W^{(j)}_x, b^{(j)}_h, b^{(j)}_x]\). 
The dimension \(d_h\) is
a design choice that we can pick on a per-problem basis. Now define the map
\(F_\theta: \mathbb{R}^{d_h} \times \mathbb{R}^d \rightarrow \mathbb{R}^{d_h}\) by composing along the first component
\[F_\theta(z,q) = F_{\theta_n}(F_{\theta_{n-1}}(\dots F_{\theta_1}(z,q), \dots q), q),q)\]
where \(\theta = [\theta_1,\dots,\theta_n]\) is a concatenation. Now suppose \(x_0,\dots,x_{T-1} \in \mathbb{R}^d\) is an observed
time series and define the dynamic
\[h_{t+1} = F_\theta(h_t, x_t)\]
up to time \(t = T\). We can think of this as a nonautonomous, 
semi-dynamical system on \(\mathbb{R}^{d_h}\) with parameter \(x = [x_0,\dots,x_{T-1}]\). Let \(\phi: \{0,\dots,T\} \times \mathbb{R}^{T \times d} \times \mathbb{R}^{d_h} \rightarrow \mathbb{R}^{d_h}\) be the solution
map generated by this difference equation. We can finally define a recurrent neural network \(\mathcal{G}(u|\cdot): \mathbb{R}^{T \times d} \rightarrow V \subseteq \mathbb{R}^{T \times d}\)
by 
\[\mathcal{G}(u|x) = \begin{bmatrix} S(A_1 \circ \phi(1,x,h_0)) \\ S(A_2 \circ \phi(2,x,h_0)) \\ \vdots \\ S(A_T \circ \phi(T,x,h_0)) \end{bmatrix}\]
where \(A_1,\dots,A_T\) are affine maps, \(S\) is a thresholding (such as softmax)
as previously discussed, and \(u\) a concatenation of the parameters \(\theta\) as well
as the parameters for all of the affine maps. Usually ones takes \(h_0 = 0\), but randomly generated initial conditions are also used in practice. 

The construction presented here is the most basic recurrent neural network. Many others architectures such as Long Short-Term Memory (LSTM), recursive, and bi-recurrent 
networks have been proposed in the literature \cite{reccurentoriginal,originallstm,anothernn,deeplearningbook}, but they are all slight modifications to the above dynamic. 
These architectures can be used as sequence to sequence maps, or, if we only consider the output at the last 
time that is \(S(A_T \circ \phi(T,x,h_0))\), as predicting \(x_T\) or classifying the sequence \(x_0,\dots,x_{T-1}\). We refer the reader to \cite{trainingrnns}
for an overview of the applications of recurrent neural networks.

\section{Algorithms}
\label{algorithms}

Subsection \ref{lossfunction} describes the choice of loss function. 
Subsection \ref{gradientbasedoptimization} outlines the state-of-the-art
derivative based optimization, with subsection \ref{iterativetechnique} 
presenting the algorithms and subsection \ref{initializationandnormalization}
presenting tricks for better convergence. Subsection \ref{ensemblekalmaninversion}
defines the EKI method, with subsequent subsections presenting our various
modifications.

\subsection{Loss Function}
\label{lossfunction}

Before delving into the specifics of optimization methods used,
we discuss the general choice of loss function \(\mathcal{L}\). While the machine learning literature contains a wide variety of loss functions that are 
designed for specific problems, there are two which are most commonly used and considered first when tackling any regression and classification
problems respectively, and on which we focus our work in this paper. 
For regression tasks, the squared-error
loss 
\[\mathcal{L}(y',y) = \|y - y'\|_\mathcal{Y}^2\] 
is standard and is well known to the inverse problems community; it arises from an additive Gaussian noise model. 
When the task at hand is classification, the standard choice of loss is the cross-entropy
\[\mathcal{L}(y',y) = - \langle y, \log y' \rangle_\mathcal{Y},\]
with the \(\log\) computed point-wise and where we consider \(\mathcal{Y} = \mathbb{R}^m\). This loss is well-defined on the space \(\mathbb{P}^m_0 \times \mathbb{P}^m\). It is consistent with the the projection map \(S\) of the neural network model being the softmax as \(\text{Im}(S) = \mathbb{P}^m_0\). A simple Lagrange multiplier argument shows that 
\(\mathcal{L}\) is indeed infimized over \(\mathbb{P}^m_0\) by 
sequence \(y' \rightarrow y\) and hence the loss is consistent with what we want our model output to be. \footnote{Note that the infimum is not, in general, 
attained in \(\mathbb{P}^m_0\) as defined, because
perfectly labeled data may take the form
\(\{y \in \mathbb{R}^m \; | \; \exists! j \text{ s.t. } y_j = 1, y_k = 0 \; \forall k \neq j\}\) which is in the closure of \(\mathbb{P}^m_0\) but not in 
\(\mathbb{P}^m_0\) itself.} From a modeling perspective,
the choice of softmax as the output layer has some drawbacks
as it only allows us to asymptotically match the data. However it is a
good choice if the cross-entropy loss is used to solve the problem; indeed, in practice, the softmax along with the cross-entropy loss has seen the best numerical results when compared to other choices of thresholding/loss pairs \cite{deeplearningbook}. 

The interpretation of the cross-entropy loss is to think of our model as approximating a categorical distribution over the input data and, to get this approximation, we want to minimize its Shannon cross-entropy with respect to the data. Note, however, that there is no additive noise model for which this loss appears in the associated MAP estimator simply because 
\(\mathcal{L}\) cannot be written purely as a function of the residual \(y - y'\).

\subsection{Gradient Based Optimization}
\label{gradientbasedoptimization}

\subsubsection{The Iterative Technique}
\label{iterativetechnique}

The current state of the art for solving optimization problems of the form (\ref{supervised_classic}), (\ref{semi_supervised_classic}), (\ref{online_classic})
is  based around the use of stochastic gradient descent (SGD) \cite{robbinsmonroe,anotherrobbinsmonroe,backprop}. We will 
describe these methods starting from a continuous time viewpoint, for 
pedagogical clarity. In particular, we think of the unknown 
parameter \(u \in \mathcal{U}\) as the large time limit of a smooth 
function of time \(u: [0, \infty) \rightarrow \mathcal{U}\). Let \(\Phi(u; \mathsf{x}, \mathsf{y}) = \Phi_{\text{s}}(u; \mathsf{x}, \mathsf{y}) \text{ or } \Phi_{\text{ss}}(u; \mathsf{x}, \mathsf{y})\) then gradient descent imposes the dynamic
\begin{equation}
\label{gradientflow}
\dot{u} = - \nabla \Phi (u; \mathsf{x}, \mathsf{y}), \quad u(0) = u_0
\end{equation}
which moves the parameter in the steepest descent direction with respect to
the regularized loss function, and hence will converge to a local minimum for Lebesgue
almost all initial data, leading to bounded trajectories \cite{lee2016gradient,stuart1998dynamical}. 

For the practical implementations of this approach in machine learning, 
a number of adaptations are made. First the ODE is discretized in time, 
typically by a forward Euler scheme; the time-step is referred to as
the \textit{learning rate}. The time-step is often, 
but not always, chosen to be a decreasing function of the
iteration step \cite{dieuleveut2016nonparametric,robbinsmonroe}. 
Secondly, at each step of the iteration, only a subset of
the data is used to approximate the full gradient. In the supervised case, for example,
\[\Phi_\text{s}(u;\mathsf{x},\mathsf{y}) \approx \frac{1}{N'} \sum_{j \in B_{N'}} \mathcal{L}(\mathcal{G}(u|x_j),y_j) + R(u)\]
where \(B_{N'} \subset \{1,\dots,N\}\) is a random subset of cardinality \(N'\) usually with \(N' \ll N\).
A new \(B_{N'}\) is drawn at each step of the Euler scheme without replacement until the full dataset has been exhausted.
One such cycle through all of the data is called an \textit{epoch}. The number of epochs it takes to train a model
varies significantly based on the model and data at hand but is usually within the range 10 to 500.
This idea, called \textit{mini-batching}, leads to the terminology 
\textit{stochastic gradient descent (SGD)}. 
Recent work has suggested that adding this type of noise helps preferentially guide
the gradient descent towards places in parameter space which generalize better
than standard descent methods \cite{sgdnoise1,sgdnoise2}.

A third variant on basic gradient descent is the
use of momentum-augmented  methods 
utilized to accelerate convergence \cite{originalnesterov}. The 
continuous time dynamic associated with the 
Nesterov momentum method is \cite{odenesterov}
\begin{align}
\label{nesterovflow}
\begin{split}
&\ddot{u} + \frac{3}{t}\dot{u} = - \nabla \Phi (u; \mathsf{x}, \mathsf{y}), \\
&u(0) = u_0, \quad \dot{u}(0)=0.
\end{split}
\end{align}
We note, however, that, while still calling it Nesterov momentum,
this is not the dynamic machine learning practitioners discretize. In fact,
what has come to be called Nesterov momentum in the machine learning literature
is nothing more than a strange discretization of a rescaled version the standard gradient flow \eqref{gradientflow}.
To see this, we note that via the approximation \(k/(k+3) \approx 1 - 3/(k+1)\) one may
discretize \eqref{nesterovflow}, as is done in \cite{odenesterov}, to obtain
\begin{align*}
u_{k+1} &= v_k - h_k \nabla \Phi(v_k; \mathsf{x}, \mathsf{y}) \\
v_{k+1} &= u_{k+1} + \frac{k}{k + 3}(u_{k+1} - u_k)
\end{align*}
with \(v_0 = u_0\) and where the sequence \(\{\sqrt{h_k}\}\) gives the step sizes. However, what machine learning practitioners use is the algorithm
\begin{align*}
u_{k+1} &= v_k - h_k \nabla \Phi(v_k; \mathsf{x}, \mathsf{y}) \\
v_{k+1} &= u_{k+1} + \lambda(u_{k+1} - u_k)
\end{align*}
for some fixed \(\lambda \in (0,1)\). This may be written as
\[u_{k+1} = (1 + \lambda)u_k - \lambda u_{k-1} - h_k \nabla \Phi((1 + \lambda)u_k - \lambda u_{k-1}; \mathsf{x}, \mathsf{y}).\]
If we rearrange and divide by \(h_k\), we can obtain
\[\frac{u_{k+1} - u_k}{h_k} = \lambda \left( \frac{u_k - u_{k_1}}{h_k} \right ) - \nabla \Phi((1 + \lambda)u_k - \lambda u_{k-1}; \mathsf{x}, \mathsf{y}) \]
which is easily seen as a discretization of a rescaled version of the gradient flow \eqref{gradientflow}, namely
\[\dot{u} = - (1 - \lambda)^{-1} \nabla \Phi(u; \mathsf{x}, \mathsf{y}).\]
However, there is a sense in which this discretization introduces
momentum, but only to order \(\mathcal{O}(h_k)\) whereas classically the momentum term would be on the order \(\mathcal{O}(1)\).
To see this, we can again rewrite the discretization as
\[u_{k+1} = 2u_k - u_{k-1} - (1 - \lambda)(u_k - u_{k-1}) - h_k \nabla \Phi((1 + \lambda)u_k - \lambda u_{k-1}; \mathsf{x}, \mathsf{y}).\]
Rearranging this and dividing through by \(h_k\) we can obtain
\[h_k \left ( \frac{u_{k+1} - 2u_k + u_{k-1}}{h_k^2} \right ) + (1 - \lambda) \left ( \frac{u_k - u_{k-1}}{h_k} \right ) = - \nabla \Phi((1 + \lambda)u_k - \lambda u_{k-1}; \mathsf{x}, \mathsf{y})\]
which may be seen as a discretization of the equation 
\[h_t \ddot{u} + (1-\lambda) \dot{u} = - \nabla \Phi(u; \mathsf{x}, \mathsf{y}),\]
where \(h_t\) is a continuous time version of the sequence \(\{h_k\}\);
in particular, if \(h_t=h \ll 1\) we see that whilst momentum is approximately
present, it is only a small effect.

From these variants on continuous time gradient descent
have come a plethora of adaptive first-order optimization methods 
that attempt to solve the learning problem. Some of the more
popular include Adam, RMSProp, and Adagrad \cite{adam,adagrad}. There is no consensus on which methods performs best although some recent work has argued in favor of SGD and 
momentum SGD \cite{marginaladaptive}.

Lastly, the online learning problem (\ref{online_classic}) is also commonly solved via a gradient descent method dubbed online gradient descent (OGD). The dynamic is
\begin{align*}
&\dot{u}_j = - \nabla \Phi_{\text{o}} (u_j, u_{j-1}; x_j, y_j) \\
&u_j(0) = u_{j-1}
\end{align*}
which can be extended to the momentum case in the obvious way. It is common that only a single step of the Euler scheme is computed. The process of letting 
all these ODE(s) evolve in time is called \textit{training}.

\subsubsection{Initialization and Normalization} 
\label{initializationandnormalization}

Two major challenges for the iterative methods presented here are: 1) finding a good starting point \(u_0\) for the dynamic, 
and 2) constraining the distribution of the outputs of each map \(F_k\) 
in the neural network. The first is usually called \textit{initialization},
while the second is termed \textit{normalization}; the two, as we will see, are related.

Historically, initialization was first dealt with using a technique called \textit{layer-wise pretraining} \cite{pretraining}. In this approach the parameters are initialized 
randomly. Then the parameters of all but first layer are held fixed and SGD is used to find the parameters of the first layer. Then all but the parameters of the second layer are held 
fixed and SGD is used to find the parameters of the second layer. Repeating this for all layers yields an estimate \(u_0\) for all the parameters, and this is then used
as an initialization for SGD in a final step called 
\textit{fine-tuning}. Development of new activation functions, namely the ReLU, has allowed simple random initialization (from a carefully designed prior measure) to work 
just as well, making layer-wise pretraining essentially obsolete. There are many proposed strategies in the literature for how one should design this prior \cite{xavier,goodinit}. The main idea behind all of them is to somehow normalize the output mean and variance of each map \(F_k\). One constructs the product probability measure
\[\mu_0 = \mu_0^{(0)} \otimes \mu_0^{(1)} \otimes \cdots \otimes \mu_0^{(n-1)} \otimes \mu_0^{(n)} \]
where each \(\mu_0^{(k)}\) is usually a centered, isotropic probability measure with covariance scaling \(\gamma_k\). Each such measure corresponds to the distribution of the parameters of each respective layer with \(\mu_0^{(n)}\) attached to the parameters of the map \(A\). A common strategy called Xavier initialization \cite{xavier}
proposes that the inverse covariance (precision) is determined  by the
average of the input and output dimensions of each layer:
\[\gamma_k^{-1}=\frac12\bigl(d_k + d_{k+1}\bigr)\]
thus
\[\gamma_k = \frac{2}{d_k + d_{k+1}}.\]
When the layer is convolutional, \(d_k\) and \(d_{k+1}\) are instead taken to be 
the number of input and output channels respectively.
Usually each \(\mu_0^{(k)}\) is then taken to be a centered Gaussian or uniform
probability measure. Once this prior is constructed one initializes SGD by a single draw.

As we have seen, initialization strategies aim to normalize the output distribution of each layer. However, once SGD starts and the parameters change,
this normalization is no longer in place. This issue has been called the \textit{internal covariate shift}. To address it, normalizing parameters 
are introduced after the output of each layer. The most common strategy for finding these parameters is called \textit{batch-normalization} \cite{batchnorm}, which, as
the name implies, relies on computing a mini-batch of the data. To illustrate the idea, suppose \(z_m(x_{k_1}), \dots, z_m(x_{k_B})\) are the outputs 
of the map \(F_{m-1}\) at inputs \(x_{k_1},\dots,x_{k_B}\). We compute the mean and variance
\[\nu_m = \frac{1}{B} \sum_{j=1}^{B} z_m(x_{k_j}); \quad \sigma^2_m = \frac{1}{B} \sum_{j=1}^B \|z_m(x_{k_j}) - \nu_m\|_2^2\]
and normalize these outputs so that the inputs to the map \(F_m\) are
\[\frac{z_m(x_{k_j}) - \nu_m}{\sqrt{\sigma^2_m + \epsilon}} \gamma + \beta\]
where \(\epsilon > 0\) is used for numerical stability while \(\gamma, \beta\) 
are new parameters to be estimated, and are termed the \textit{scale} and \textit{shift} 
respectively; they are found by means of the SGD optimization process.
It is not necessary to introduce the new parameters \(\gamma, \beta\) but is common in practice and, with them, the operation is called \textit{affine batch-normalization}.
When an output has multiple channels, separate normalization is done per channel. During training a running mean of each \(\nu_m, \sigma^2_m\) is kept and the resulting values
are used for the final model. A clear drawback to batch normalization is that it relies on batches of the data to be computed and hence cannot be used in the online setting. Many similar strategies have been proposed \cite{intancenorm,layernorm} with no clear consensus on which works best. Recently a new activation function 
called SeLU \cite{selfnormalizing} has been claimed to perform the required normalization automatically.

\subsection{Ensemble Kalman Inversion}
\label{ensemblekalmaninversion}

The Ensemble Kalman Filter (EnKF) is a method for estimating the state of a stochastic dynamical system from noisy observations \cite{enkforiginal}.
Over the last decade the method has been systematically developed as an
iterative method for solving general inverse problems; in this context, it is
sometimes referred to as Ensemble Kalman Inversion (EKI) \cite{enkfinverse}.
Viewed as a sequential Monte Carlo method \cite{enkfanalysis}, it works on an ensemble of parameter estimates (particles) transforming them 
from the prior into the posterior. Recent work has established, however, that 
unless the forward operator is linear and the additive noise is 
Gaussian \cite{enkfanalysis},
the correct posterior is not obtained \cite{enkfwrongposterior1}. Nevertheless there is ample numerical evidence that shows EKI works very well as a derivative-free
optimization method for nonlinear least-squares problems \cite{goodenkf1,goodenkf2}.
In this paper, we view it purely through the lens of optimization and propose 
several modifications to the method that follow from adopting this perspective
within the context of machine learning problems.

Consider the general inverse problem 
\[\mathsf{y} = \mathsf{G}(u) + \eta \]
where \(\eta \sim \pi=N(0,\Gamma)\) represent noise, and let \(\mu_0\) be a prior 
measure on the parameter \(u\).
Note that the supervised, semi-supervised, 
and online learning problems (\ref{static_inverse}), (\ref{dynamic_inverse}) can be put into this general framework by adjusting 
the number of data points in the concatenations \(\mathsf{y}\), \(\mathsf{x}\) and letting \(\mathsf{x}\) be absorbed into the 
definition of \(\mathsf{G}\). Let \(\{u^{(j)}\}_{j=1}^J \subset \mathcal{U}\) be an ensemble of parameter estimates which we will allow to evolve in time
through interaction with one another and with the data; this ensemble may be initialized by drawing independent samples from \(\mu_0\), for example.
The evolution of  \(u^{(j)}: [0, \infty) \rightarrow \mathcal{U}\) is described
by the EKI dynamic \cite{enkfanalysis}
\begin{align*}
\label{enkf}
\begin{split}
& \dot{u}^{(j)} = -C^{\text{uw}}(u) \Gamma^{-1}(\mathsf{G}(u^{(j)}) - \mathsf{y}),\\
& u^{(j)}(0) = u^{(j)}_0.
\end{split}
\end{align*}
Here 
\[\bar{\mathsf{G}} = \frac{1}{J} \sum_{l=1}^J \mathsf{G}(u^{(l)}), 
\quad \bar{u} = \frac{1}{J} \sum_{l=1}^J u^{(l)}\]
and \(C^{\text{uw}}(u)\) is the empirical cross-covariance operator
\[C^{\text{uw}}(u) = \frac{1}{J} \sum_{j=1}^J (u^{(j)} - \bar{u}) \otimes (\mathsf{G}(u^{(j)}) - \bar{\mathsf{G}}).\]
Thus
\begin{align}
\label{enkf}
\begin{split}
&\dot{u}^{(j)} 
= -\frac{1}{J} \sum_{k=1}^J \langle \mathsf{G}(u^{(k)}) - \bar{\mathsf{G}}, \mathsf{G}(u^{(j)}) - \mathsf{y} \rangle_\Gamma \: u^{(k)},\\
&u^{(j)}(0) = u^{(j)}_0.
\end{split}
\end{align}

Viewing the difference of \(\mathsf{G}(u^{(k)})\) from its mean, appearing
in the left entry of the inner-product, as a projected approximate derivative 
of \(\mathsf{G}\), it is possible to understand \eqref{enkf} as an
approximate gradient descent.

Rigorous analysis of the long-term properties of this dynamic 
for a finite \(J\) are poorly understood except in the case where \(\mathsf{G}(\cdot) = A \cdot\) is linear \cite{enkfanalysis}.
In the linear case, we obtain that \(u^{(j)} \rightarrow u^*\) as \(t \rightarrow \infty\) where \(u^*\) minimizes the functional
\[\Phi(u; \mathsf{y}) = \frac{1}{2} \|\mathsf{y} - Au\|_\Gamma^2\]
in the subspace \(\mathcal{A} = \text{span} \{u^{(j)}_0-\bar{u}\}_{j=1}^J\),
and where \(\bar{u}\) is the mean of the initial ensemble \(\{u^{(j)}_0\}\). 
This follows from the fact that, in the linear case,
we may re-write (\ref{enkf}) as
\[\dot{u}^{(j)} = - C(u) \nabla_u \Phi(u^{(j)}; \mathsf{y})\]
where \(C(u)\) is an empirical covariance operator
\[C(u) = \frac{1}{J} \sum_{j=1}^J (u^{(j)} - \bar{u}) \otimes (u^{(j)} - \bar{u}).\]
Hence each particle performs a gradient descent with respect to \(\Phi\)
and \(C(u)\) projects into the subspace \(\mathcal{A}\). 

To understand the nonlinear setting we use linearization.
Note from (\ref{enkf}) that
\begin{align*}
\dot{u}^{(j)} &= - \frac{1}{J} \sum_{k=1}^J \langle \mathsf{G}(u^{(k)}) - \frac{1}{J} \sum_{l=1}^J \mathsf{G}(u^{(l)}),
\mathsf{G}(u^{(j)}) - \mathsf{y} \rangle_\Gamma \: u^{(k)} \\
&= - \frac{1}{J} \sum_{k=1}^J \langle \mathsf{G}(u^{(k)}) - \frac{1}{J} \sum_{l=1}^J \mathsf{G}(u^{(l)}),
\mathsf{G}(u^{(j)}) - \mathsf{y} \rangle_\Gamma \: (u^{(k)}-\bar{u})\\
&= - \frac{1}{J^2} \sum_{k=1}^J \sum_{l=1}^J \langle \mathsf{G}(u^{(k)}) - \mathsf{G}(u^{(l)}), 
\mathsf{G}(u^{(j)}) - \mathsf{y} \rangle_\Gamma \: (u^{(k)}-\bar{u}).
\end{align*}
Now we linearize on the assumption that the particles are close to
one another, so that
\begin{align*}
\mathsf{G}(u^{(k)}) = \mathsf{G}(u^{(j)} + u^{(k)} - u^{(j)}) &\approx \mathsf{G}(u^{(j)}) + D\mathsf{G}(u^{(j)})
(u^{(k)} - u^{(j)}) \\
\mathsf{G}(u^{(l)}) = \mathsf{G}(u^{(j)} + u^{(l)} - u^{(j)}) &\approx \mathsf{G}(u^{(j)}) + D\mathsf{G}(u^{(j)})
(u^{(l)} - u^{(j)}).
\end{align*}
Here \(D\mathsf{G}\) is the Fr\'echet derivative of \(\mathsf{G}\). With this approximation, we obtain
\begin{align*}
\dot{u}^{(j)} &\approx -\frac{1}{J^2} \sum_{k=1}^J \sum_{l=1}^J \langle D \mathsf{G}^* (u^{(j)})(\mathsf{G}(u^{(j)}) - \mathsf{y}), u^{(k)} - u^{(l)} \rangle_\Gamma (u^{(k)} - \bar{u}) \\
&= - \frac{1}{J} \sum_{k=1}^J \langle D \mathsf{G}^* (u^{(j)})(\mathsf{G}(u^{(j)}) - \mathsf{y}), u^{(k)} - \bar{u} \rangle_\Gamma (u^{(k)} - \bar{u}) \\
&= -C(u) \nabla_u \Phi (u^{(j)}, \mathsf{y})
\end{align*}
where
\[\Phi(u; \mathsf{y}) = \frac{1}{2} \|\mathsf{y} - \mathsf{G}(u)\|_\Gamma^2.\]
This is again just gradient descent with a projection onto the subspace \(\mathcal{A}\). These arguments also motivate the 
interesting variants on EKI proposed in \cite{haber}; indeed the paper
\cite{haber} inspired the organization of the linearization calculations
above.

In summary, the EKI is a methodology which behaves like
gradient descent, but achieves this without computing gradients. Instead
it uses an ensemble and is hence inherently parallelizable. In the context
of machine learning this opens up the possibility of avoiding explicit
backpropagation, and doing so in a manner which is well-adapted to emerging
computer architectures.

\subsubsection{Cross-Entropy Loss}
The previous considerations demonstrate that EKI as typically used
is closely related to minimizing an \(\ell_2\) loss function via
gradient descent. Here we propose a simple modification to the method allowing it to minimize any 
loss function instead of only the squared-error; our primary
motivation is the case of cross-entropy loss. 

Let \(\mathcal{L}(\mathsf{y}', \mathsf{y})\) be any loss function, this may, for example, be the cross entropy
\[\mathcal{L}(\mathsf{y}', \mathsf{y}) = -\frac{1}{N} \langle \mathsf{y}, \log \mathsf{y}' \rangle_{\mathcal{Y}^N}.\]
Now consider the dynamic
\begin{align}
\label{anyloss_enkf}
\begin{split}
\dot{u}^{(j)} &= - C^{\text{uw}}(u) \nabla_{\mathsf{y}'} \mathcal{L}(\mathsf{G}(u^{(j)}), \mathsf{y}) \\
&= -\frac{1}{J} \sum_{k=1}^J \langle \mathsf{G}(u^{(k)}) - \bar{\mathsf{G}}, \nabla_{\mathsf{y}'} \mathcal{L}(\mathsf{G}(u^{(j)}), \mathsf{y}) \rangle \: u^{(k)}.
\end{split}
\end{align}
If \(\mathcal{L}(\mathsf{y}', \mathsf{y}) = \frac{1}{2} \|\mathsf{y} - \mathsf{y}'\|_\Gamma^2\) then \(\nabla_{\mathsf{y}'} \mathcal{L}(\mathsf{G}(u^{(j)}), \mathsf{y}) = \Gamma^{-1}(\mathsf{G}(u^{(j)}) - \mathsf{y})\)
recovering the original dynamic. Note that since we've defined the loss through the auxiliary variable \(\mathsf{y}'\) which is meant to stand-in for the output of our model, the 
method remains derivative-free with respect to the model parameter \(u\), but does not allow for adding regularization directly into the loss. However regularization could be 
added directly into the dynamic; we leave such considerations for future work. 

An interpretation of the original method is that it aims to make the norm of the residual \(\mathsf{y} - \mathsf{G}(u^{(j)})\) small.
Our modified version replaces this residual with \(\nabla_{\mathsf{y}'} \mathcal{L}(\mathsf{G}(u^{(j)}), \mathsf{y})\),
but when \(\mathcal{L}\) is the cross entropy this is in fact the same (in the \(\ell_1\) sense).
We make this precise in the following
proposition.

\begin{theorem}
Let \(\mathsf{G}: \mathcal{U} \rightarrow (\mathbb{P}^m_0)^N\) and suppose \(\mathsf{y} = [e_{k_1},\dots,e_{k_N}]^T\) where 
\(e_{k_j}\) is the \(k_j\)-th standard basis vector of \(\mathbb{R}^m\).
Then \(u^* \in \mathcal{U}\) is a solution to 
\[\argmin_{u \in \mathcal{U}} \|\mathsf{y} - \mathsf{G}(u)\|_{\ell_1}\]
if and only if \(u^*\) is a solution to 
\[\argmin_{u \in \mathcal{U}} \|\nabla_{\mathsf{y}'} \mathcal{L}(\mathsf{G}(u), \mathsf{y})\|_{\ell_1}\]
where \(\mathcal{L}(\mathsf{y}',\mathsf{y}) = - \langle \mathsf{y}, \log \mathsf{y}' \rangle_{\ell_2}\) is the cross-entropy loss.
\end{theorem}
\begin{proof}
Without loss of generality, we may assume \(N=1\) and thus let \({\mathsf{y} = e_k}\) be the \(k\)-th standard basis vector 
of \(\mathbb{R}^m\). Suppose that \(u^*\) is a solution to \({\argmin_{u \in \mathcal{U}} \|\mathsf{y} - \mathsf{G}(u)\|_{\ell_1}}\).
Then for any \(u \in \mathcal{U}\), we have
\[\sum_{j \neq k} \mathsf{G}(u^*)_j + (1- \mathsf{G}(u^*)_k) \leq \sum_{j \neq k} \mathsf{G}(u)_j + (1- \mathsf{G}(u)_k).\]
Adding \(0 = \mathsf{G}(u^*)_k - \mathsf{G}(u^*)_k\) to the l.h.s. and \(0 = \mathsf{G}(u)_k - \mathsf{G}(u)_k\) to the 
r.h.s. and noting that \(\|\mathsf{G}(u)\|_{\ell_1} = 1\) for all \(u \in \mathcal{U}\) since \(\text{Im}(\mathsf{G}) = \mathbb{P}_0^m\)
we obtain
\[2(1 - \mathsf{G}(u^*)_k) \leq 2(1 - \mathsf{G}(u)_k)\]
which implies
\[\frac{1}{\mathsf{G}(u^*)_k} \leq \frac{1}{\mathsf{G}(u)_k}\]
as required since \(\|\nabla_{\mathsf{y}'} \mathcal{L}(\mathsf{G}(u), \mathsf{y})\|_{\ell_1} = 1/\mathsf{G}(u)_k\).
The other direction follows similarly.
\end{proof}

\subsubsection{Momentum} 
Continuing in the spirit of optimization, we may also add Nesterov momentum to the EKI method. This is a simple modification 
to the dynamic \eqref{anyloss_enkf},
\begin{align}
\begin{split}
\label{momentum_enkf}
&\ddot{u}^{(j)} + \frac{3}{t} \dot{u}^{(j)} = - C^{\text{uw}}(u) \nabla_{\mathsf{y}'} \mathcal{L}(\mathsf{G}(u^{(j)});\mathsf{y}) \\
&u^{(j)}(0) = u_0^{(j)}, \quad \dot{u}^{(j)}(0) = 0.
\end{split}
\end{align}
While we present momentum EKI in this form, in practice, we follow the standard in machine learning by
fixing a momentum factor \(\lambda \in (0,1)\) and discretizing \eqref{anyloss_enkf} 
using the method shown in subsection \ref{iterativetechnique}. 
In standard stochastic gradient decent, it has been observed that this discretization converges more quickly 
and possibly to a better local minima than the forward Euler discretization \cite{importanceofinitmom}. 
Numerically, we discover a similar speed up for EKI. However, the memory cost doubles 
as we need to keep track of an ensemble of positions and momenta.
Some experiments in the next section demonstrate the speed-up effect.
We leave analysis and possible applications to other inverse problems of the 
momentum method as presented in \eqref{momentum_enkf} for future work.

\subsubsection{Discrete Scheme} 
\label{discretescheme}
Finally we present our modified EKI method in the implementable, discrete time setting and discuss some variants on this basic
scheme  which are particularly useful for machine learning problems. In implementation, it is useful to consider the concatenation of particles \(\mathsf{u} = [u^{(1)}, \dots, u^{(J)}]\)
which may be viewed as a function \(\mathsf{u} : [0,\infty) \rightarrow \mathcal{U}^J\). Then (\ref{anyloss_enkf}) becomes
\[\dot{\mathsf{u}} = - D(\mathsf{u}) \mathsf{u}\]
where for each fixed \(u\) the operator
\(D(\mathsf{u}) : \mathcal{U}^J \rightarrow \mathcal{U}^J\) is a linear operator. Suppose \(\mathcal{U} = \mathbb{R}^P\) then we may exploit symmetry 
and represent \(D(\mathsf{u})\) by a \(J \times J\) matrix instead of a \(JP \times JP\) matrix. To this end, suppose the ensemble members are stacked 
row-wise that is \(\mathsf{u} \in \mathbb{R}^{J \times P}\) then \(D(\mathsf{u})\) has the simple representation
\[(D(\mathsf{u}))_{kj} = \langle \mathsf{G}(u^{(k)}) - \bar{\mathsf{G}}, \nabla_{\mathsf{y}'} \mathcal{L}(\mathsf{G}(u^{(j)}), \mathsf{y}) \rangle\]
which is readily verified by (\ref{anyloss_enkf}).
We then discretize via an adaptive forward Euler scheme to obtain
\[\mathsf{u}_{k+1} = \mathsf{u}_k - h_k D(\mathsf{u}_k)\mathsf{u}_k.\]
Choosing the correct time-step has an immense impact on practical performance. We have found that the choice
\[h_k = \frac{h_0}{\|D(\mathsf{u}_k)\|_F + \epsilon},\]
where \(\|\cdot\|_F\) denotes the Frobenius norm,
works well in practice \cite{mattprivate}. We aim to make \(h_0\) as large as possible without loosing stability of the dynamic. 
The intuition behind this choice has to do with the fact that that \(D(\mathsf{u})\) measures 
how close the propagated particles are to each other (left part of the inner-product) and how close they are to the 
data (right part of the inner-product). When either or both of these are small, we may take larger steps, and still retain numerical stability,
by choosing \(h_k\) inversely 
proportional to \(\|D(\mathsf{u})\|_F\); 
the parameter \(\epsilon\) is added 
to avoid floating point issues when 
\(\|D(\mathsf{u})\|_F\) is near machine
precision.
As \(k \rightarrow \infty\), we typically match the data with increasing
accuracy and, simultaneously, the propagated particles achieve consensus
and collapse on one another; as a consequence
\(\|D(\mathsf{u}_k)\|_F \rightarrow 0\) which means we take larger and larger steps. Note that this is in contrast to the Robbins-Monro implementation
of stochastic gradient descent where the 
sequence of time-steps are chosen to decay monotonically to zero.

Similarly, the momentum discretization of \eqref{anyloss_enkf} is
\begin{align*}
\mathsf{u}_{k+1} &= \mathsf{v}_k - h_k D(\mathsf{v}_k) \mathsf{v}_k \\
\mathsf{v}_{k+1} &= \mathsf{u}_{k+1} + \lambda (\mathsf{u}_{k+1} - \mathsf{u}_k)
\end{align*}
with \(\lambda \in (0,1)\) fixed, \(\mathsf{u}_0 = \mathsf{v}_0\) where \(h_k = h_0/(\|D(\mathsf{v}_k)\|_F + \epsilon)\) as before 
and \(\mathsf{v}\) represent the particle momenta.

We now present a list of numerically successful heuristics that we employ when solving 
practical problems. 

\renewcommand{\labelenumi}{(\Roman{enumi})}
\renewcommand{\labelenumii}{\roman{enumii}.}
\begin{enumerate}
    \item \textbf{Initialization}: To construct the initial ensemble, we draw an i.i.d. sequence \(\{u_0^{(j)}\}_{j=1}^J\)
    with \(u_0^{(1)} \sim \mu_0\) where \(\mu_0\) is selected according to the construction discussed for initialization of the neural network model
    in the section outlining SGD.
    \item \textbf{Mini-batching}: We borrow from SGD the idea of mini-batching where we use only a subset of the data 
    to compute each step of the discretized scheme, picking randomly without replacement. As in the classical SGD context, we call a cycle through 
    the full dataset an epoch.
    \item \textbf{Prediction}: In principle, any one of the particles \(u^{(j)}\) can be used as the parameters of 
    the trained model. However, as analysis of Figure \ref{DenseParticles}
    below shows, the spread in their performance is quite small; furthermore 
    even though the system is nonlinear, the mean particle \(\bar{u}\) 
achieves an equally good performance as the individual particles.
    Thus, for computational simplicity, we choose to use the mean particle as our final parameter estimate.
    This choice further motivates one of the ways in which we randomize.
    \item \textbf{Randomization}: The EKI property that 
all particles remain in the subspace spanned by the initial ensemble is not
desirable when \(J \ll \dim \mathcal{U}\).
    We break this property by introducing noise into the system. We have found two numerically successful ways 
    of accomplishing this.
    \begin{enumerate}
        \item At each step of the discrete scheme, add noise to each particle,
        \[u^{(j)}_k \mapsto u^{(j)}_k + \eta_k^{(j)}\]
        where \(\{\eta_k^{(j)}\}_{j=1}^J\) is an i.i.d. sequence with \(\eta_k^{(1)} \sim \mu_k\). We define 
        \(\mu_k\) to be a scaled version of \(\mu_0\) by scaling its covariance operator namely \(C_k = \sqrt{h_k} C_0\),
        where \(h_k\) is the time step as previously defined. Note that as the particles start to collapse,
        \(h_k\) increases, hence we add more noise to counteract this. In the momentum case, we perform the same mapping
        but on the particle momenta instead
        \[v^{(j)}_k \mapsto v^{(j)}_k + \eta_k^{(j)}.\]
        \item At the end of each epoch, randomize the particles around their mean,
        \[u^{(j)}_{kT} \mapsto \bar{u}_{kT} + \eta^{(j)}_{kT}\]
        where \(T\) is the number of steps needed to complete a cycle through the entire dataset and  
        \(\{\eta_{kT}^{(j)}\}_{j=1}^J\) is an i.i.d. sequence with \(\eta_{kT}^{(1)} \sim \mu_0\). Note that because 
        this randomization is only done after a full epoch, it is not clear how the noise should be scaled and 
        thus we simply use the prior. This may not be the optimal thing to do, but we have found great numerical success
        with this strategy. Figure \ref{NoiseSpread} shows the spread of the ratio of the parameters to the the noise \(\|\bar{u}_{kT}\|/\|\eta^{(j)}_{kT}\|\).
        We see that relatively less noise is added as training continues. 
        It may be possible to achieve better results
        by increasing the noise with time as to combat collapse. However, we do not perform such experiments.
        Furthermore we have found that this does not work well in 
the momentum case; hence all randomization for the momentum scheme 
is done according to the first point.
    \end{enumerate}
    \item \textbf{Expanding Ensemble}: Numerical experiments show that using a small number of particles
    tends to have very good initial performance (one to two epochs) that quickly saturates. On the other hand,
    using a large number of particles does not do well to begin with but greatly outperforms small particle ensembles
    in the long run. Thus we use the idea of an expanding ensemble where we gradually add in new particles.
    This is done in the context of point (ii.) of the randomization section. Namely, at the end of an epoch,
    we compute the ensemble mean and create a new larger ensemble by randomizing around it.  
\end{enumerate}

Lastly we mention that, in many inverse problem applications, it is good practice to randomize the data for each particle at each step \cite{dataassim} namely map
\[\mathsf{y} \mapsto \mathsf{y} + \xi^{^{(j)}}_k\]
where \(\{\xi^{(j)}_k\}_{j=1}^J \) is an i.i.d. sequence with \(\xi^{(1)}_k \sim \pi\). However we have found that this does not work well for classification problems.
This may be because the given classifications are correct and there is no actual noise in the data. Or it may be that we 
simply have not found a suitable measure from which to draw noise. We have not experimented in the case where the labels are 
noisy and leave this for future work.

\section{Numerical Experiments}
\label{numericalexperiments}

In the following set of experiments, we demonstrate the wide applicability of EKI on several 
machine learning tasks. All forward models we consider are some type of neural network, except for 
the semi-supervised learning case where we consider the construction in Example \ref{ex:2.2}. We benchmark 
EKI against SGD and momentum SGD and do not consider any other first-order adaptive methods.
Recent work has shown that their value is only marginal and the solutions they find may 
not generalize as well \cite{marginaladaptive}. Furthermore we do not employ batch normalization as it is not clear 
how it should be incorporated with EKI methods. However, when batch normalization is necessary, we 
instead use the SELU nonlinearity, finding the performance to be 
essentially identical to batch normalization on problems where we have
been able to compare.

The next five subsections are organized as follows. Subsection \ref{numericsconclusions}
contains the conclusions drawn from the experiments. In subsection \ref{numericsdatasets}, 
we describe the five data sets used in all of our experiments as well as the metrics used to evaluate the methods.
Subsection \ref{numericsimplementation} gives implementation details and assigns methods using different techniques their own name. 
In subsections \ref{numericssupervised}, \ref{numericssemisupervised}, and \ref{numericsonline} we show the supervised, semi-supervised, and online learning 
experiments respectively. Since most of our experiments are supervised, we split subsection \ref{numericssupervised} 
based on the type of model used namely dense neural networks, convolutional neural networks,
and recurrent neural networks respectively.

\subsection{Conclusions From Numerical Experiments}
\label{numericsconclusions}

The conclusions of our experiments are as follows:

\begin{itemize}
\item On supervised classification problems with a feed-forward neural network, EKI performs just as well as SGD even 
when the number of unknown parameters is very large (up to half a million) and the number of
ensemble members is considerably smaller (by two orders of magnitude). Furthermore EKI seems
more numerically stable than SGD, as seen in the smaller amount of oscillation in the test accuracy, 
and requires less hyper-parameter tuning. In fact, the only parameter we vary in our experiments
is the number of ensemble members, and we do this simply to demonstrate its effect. However due to the large
number of forward passes required at each EKI iteration, we have found the method to be significantly 
slower. This issue can be mitigated if each of the forward computations is 
parallelized across multiple processing units, as it often is in many industrial applications \cite{parallelenkf3,parallelenkf1,parallelenkf2}.
We leave such computational considerations for future work, as our current goal is simply 
to establish proof of concept. These experiments can be
found in the first two subsections of section \ref{numericssupervised}.

\item On supervised classification problems with a 
recurrent neural network, EKI significantly outperforms SGD.
This is likely due to the steep
barriers that occur on the loss surface of recurrent networks \cite{hardrnn1,hardrnn2} which EKI may be able
to avoid due to its noisy Jacobian estimates. These experiments can be found
in the last subsection of section \ref{numericssupervised}. 

\item On the semi-supervised learning 
problem we consider, EKI does not perform as well as 
state of the art (MCMC) \cite{semisupervised}, but performs better 
than the naive solution. However, even with 
a large number of ensemble members, EKI is much faster and computationally cheaper 
than MCMC, allowing applications to large scale problems. These experiments can be found
in section \ref{numericssemisupervised}.

\item On online regression problems tackled with a recurrent neural network, 
EKI converges significantly faster and to a better solution than SGD with \(\mathcal{O}(1)\)
ensemble members. While the problems we consider are only simple, univariate time-series,
the results demonstrate great promise for harder problems. It has long been known that 
recurrent neural networks are very hard to optimize with gradient-based techniques \cite{hardrnn1},
so we are very hopeful that EKI can improve on current state of the art. Again, we leave 
such domain specific applications to future work. These experiments can be found
in section \ref{numericsonline}.

\end{itemize}

\subsection{Data Sets}
\label{numericsdatasets}

We consider three data sets where the problem at hand is classification and two data sets
where it is regression. For classification, two of the data sets are comprised of images 
and the third of voting histories. Our goal is to classify the image based on its content 
or classify the voting record based on party affiliation. For regression, both 
datasets are univariate time-series and our goal is to predict an unobserved part of the
series. Figure \ref{DataSamples} shows samples from each of the data sets.

\begin{figure}[t]
    \centering
    \begin{subfigure}[b]{0.19\textwidth}
        \includegraphics[width=\textwidth]{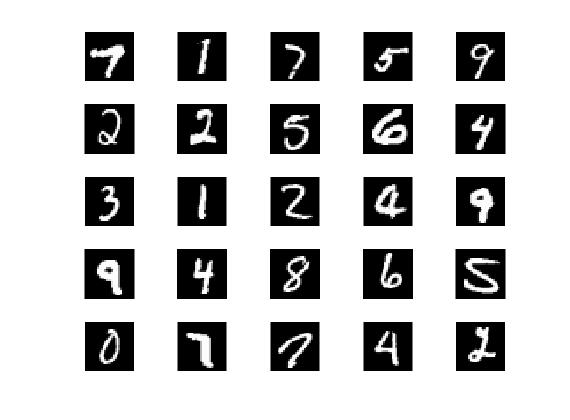}
    \end{subfigure}
    \begin{subfigure}[b]{0.19\textwidth}
        \includegraphics[width=\textwidth]{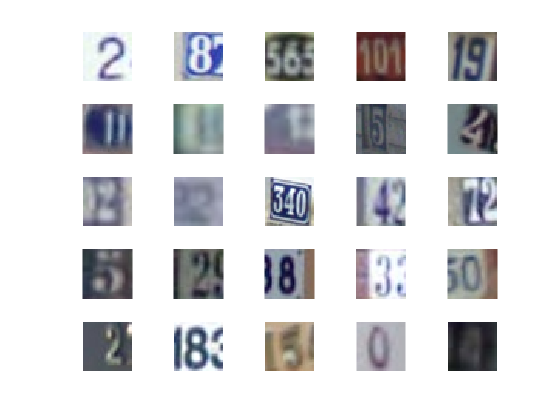}
    \end{subfigure}
    \begin{subfigure}[b]{0.19\textwidth}
        \includegraphics[width=\textwidth]{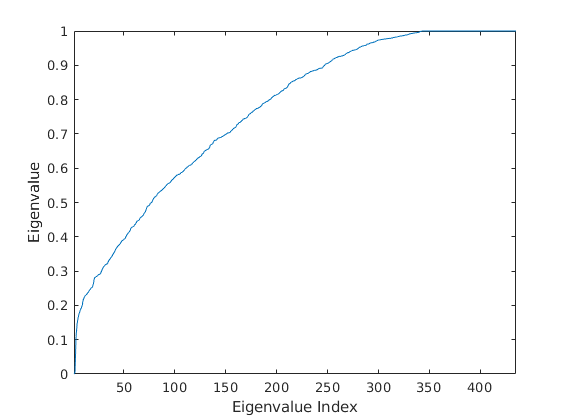}
    \end{subfigure}
    \begin{subfigure}[b]{0.19\textwidth}
        \includegraphics[width=\textwidth]{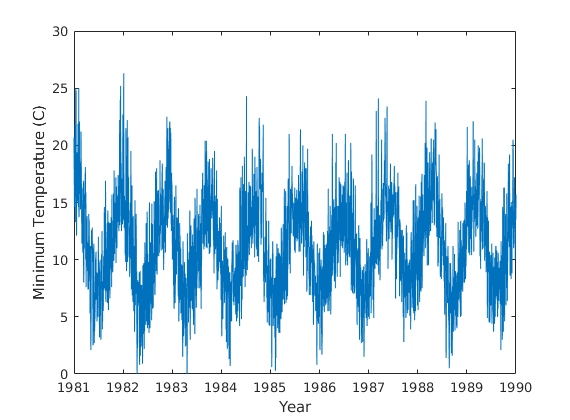}
    \end{subfigure}
    \begin{subfigure}[b]{0.19\textwidth}
        \includegraphics[width=\textwidth]{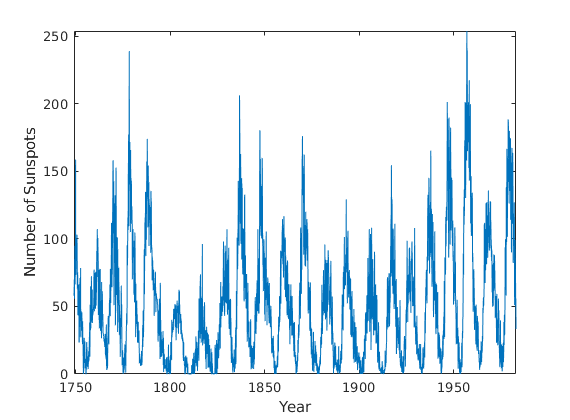}
    \end{subfigure}

    \caption{The five data sets used in numerical experiments. From left to right, the first are 25 samples
    from MNIST and SVHN respectively. The third shows the spectrum of graph Lalpacian for the Voting Records
    data set. The last two are the full time-series for the daily minimum temperatures in Melbourne and 
    the monthly number of sunspots from Z{\"u}rich respectively.}
    \label{DataSamples}
\end{figure}

As outlined in section \ref{learningproblem}, the goal of learning is to find a model which generalizes 
well to unobserved data. Thus, to evaluate this criterion, we split all data sets 
into a training and a testing portion. The training portion is used when we let our ODE(s)
evolve in time as described in section \ref{algorithms}. The testing portion is used only to evaluate 
the model. In other contexts, the training set is further split to create a validation set,
but, since we perform no hyper-parameter tuning, we omit this step. For classification,
the metric we use is called \textit{test accuracy}. This is the total number of correctly classified
examples divided by the total number of examples in the test set. For regression,
the metric we use is called \textit{test error}. This is the average (across the test set) squared \(\ell_2\)-norm 
of the difference between the true value and our prediction.

\subsubsection{Classification} The first data set we consider is MNIST \cite{mnist}. It contains
70,000 images of hand-written digits. All examples are \(28 \times 28 \) grayscale images and each 
is given a classification in \(\{0,\dots,9\}\) depending on what digit appears in the image.
Thus we consider \(\mathcal{X} = \mathbb{R}^{28 \times 28} \cong \mathbb{R}^{784}\) and 
\(\mathcal{Y} = \mathbb{P}^{10}\). Each of the labels \(y_j\) is a standard basis vector 
of \(\mathbb{R}^{10}\) with the position of the 1 indicating the digit. We use 60,000
of the images for training and 10,000 for testing. Since grayscale values range from
0 to 255, all images are fist normalized to the range \([0,1]\) by point-wise dividing
by 255. Treating all training images as a sequence of \(60000 \cdot 784\) numbers,
their mean and standard deviation are computed. Each image (including the test set) is then again normalized
via point-wise subtraction by the mean and point-wise division by the standard deviation. 
This data normalization technique is standard in machine learning.

The second image data set we consider is called SVHN \cite{svhn}. It contains 99,289 natural 
images of cropped house numbers taken from Google Street View. All examples are \(32 \times 32\)
RGB images and each  is given a classification in \(\{0,\dots,9\}\) depending on what 
digit appears in the image. Thus we consider \(\mathcal{X} = \mathbb{R}^{3 \times 32 \times 32} \cong \mathbb{R}^{3072}\) and 
\(\mathcal{Y} = \mathbb{P}^{10}\) with the labels again being basis vectors of \(\mathbb{R}^{10}\). We use
73,257 of the images for training and 26,032 for testing. All values are first normalized to be 
in the range \([0,1]\). We then perform the same normalization as in MNIST, but this time per channel.
That is, for all training images, we treat each color channel as a sequence of \(73257 \cdot 1024\)
numbers, compute the mean and standard deviation then normalize each channel as before.

The last data set for classification we consider contains the voting record of the 435
U.S. House of Representatives members; see \cite{andreasemi} and references therein. 
The votes were recorded in 1984 from the \nth{98} United States Congress, \nth{2}
session. Each record is tied to a particular representative and is a vector in \(\mathcal{X} = \mathbb{R}^{16}\)
with each entry being \(+1\), \(-1\), or \(0\) indicating a vote for, against, or abstain respectively.
The labels live in \(\mathcal{Y} = \mathbb{R}\) and are \(+1\) or \(-1\) indicating Democrat or Republican respectively.
We use this data set only for semi-supervised learning and thus pick the amount of 
observed labels \(|Z'| = 5\) with 2 Republicans and 3 Democrats. No normalization is performed.
When computing the test accuracy, we do so over the entire data sets namely we do not remove 
the 5 observed records.

\subsubsection{Regression} The first data set we consider for regression is a time series 
of the daily minimum temperatures (in Celsius) in Melbourne, Australia from January \nth{1} 1981 
to December \nth{31} 1990 \cite{melbournetemps}. It contains 3650 total observations of which we use the first
3001 for training (up to March \nth{22} 1989) and the rest for testing. We consider \(\mathcal{X} = \mathcal{Y} = \mathbb{R}\)
by letting (in the training set) the data be the first 3000 observations and the labels be the \nth{2} to \nth{3001} observations
i.e. a one-step-ahead split. The same is done for the testing set.
The minimum and 
maximum values \(x_\text{min}\), \(x_\text{max}\) over the training set are computed and 
all data is transformed via 
\[x_j \mapsto \frac{x_j - x_\text{min}}{x_\text{max} - x_\text{min}}.\]
This ensures the training set is in the range \([0,1]\) and the testing set will also be close to that range.

The second data set for regression is a time series containing the number of observed sunspots 
from Z{\"u}rich, Switzerland during each month from January 1749 to December 1983 \cite{sunspots}. It 
contains 2820 observations of which we use the first 2301 for training (up to September 1915) and 
the rest for testing. 
The data is treated in exactly the same way as the temperatures data set.

\subsection{Implementation Details}
\label{numericsimplementation}

Having outlined many different strategies for performing EKI , we give methods
using different techniques their own name so they are easily distinguishable. We refer to the techniques
listed in section \ref{discretescheme}. All methods are initialized in the same way (with the prior constructed based on the model)
and all use mini-batching. We refer to the forward Euler discretization of equation \eqref{anyloss_enkf} as EKI and
the momentum discretization, presented in section \ref{iterativetechnique}, of equation \eqref{anyloss_enkf} as MEKI. When randomizing around the mean
at the end of each epoch, we refer to the method as EKI(R). When randomizing the momenta at each 
step, we refer to the method as MEKI(R). Similarly, we call momentum SGD, MSGD.
All methods use the time step described in section \ref{discretescheme}
with hyper-parameters \(h_0 = 2\) and \(\epsilon = 0.5\) fixed. For any classification problem (except the Voting Records data set),
all methods use the cross-entropy loss whose gradient is implemented with a slight correction 
for numerical stability. Namely, in the case of a single data point, we implement
\[(\nabla_{y'} \mathcal{L}(\mathcal{G}(u),y))_k = -\frac{y_k}{(\mathcal{G}(u))_k + \delta}\]
where the constant \(\delta:=0.005\) is fixed for all our numerical
experiments. Otherwise the mean squared-error loss is used.
All implementations are done using the PyTorch framework \cite{pytorch} on a single machine 
with an NVIDIA GTX 1080 Ti GPU.

\subsection{Supervised Learning}
\label{numericssupervised}

\subsubsection{Dense Neural Networks}

In this section, we benchmark all of our proposed methods on the MNIST problem using 
four dense neural networks of increasing complexity. The four network architectures are 
outlined in Figure \ref{DenseNets}. This will allow us to compare the methods
and pick a front runner for later experiments.

\begin{figure}[t]
\centering 
    \begin{tabular}{| c | c | c |}
        \hline
        \multicolumn{3}{ |c| }{Dense Neural Networks} \\
        \hline
        Name & Architecture & Parameters \\ \hline
        DNN 1 & 784-10 & 7,850 \\ \hline
        DNN 2 & 784-100-10 & 79,510 \\ \hline
        DNN 3 & 784-300-100-10 & 266,610 \\ \hline
        DNN 4 & 784-500-300-100-10 & 573,910 \\ \hline
     \end{tabular}
     \caption{Architectures of the four dense neural networks considered.
     All networks use a softmax thresholding and a ReLU nonlinearity.}
     \label{DenseNets}
\end{figure}

We fix the ensemble size of all methods to \(J = 2000\) and the batch size to 600. SGD uses a learning rate of \(0.1\)
and all momentum methods use the constant \(\lambda = 0.9\). Figure \ref{DenseAccuracies} shows 
the final test accuracies for all methods while Figure \ref{DenseGrowth} shows the accuracies at the end of each
epoch. Due to memory constrains, we do not implement MEKI for DNN-(3,4). 
In general momentum SGD performs best, but EKI(R) trails closely. The momentum EKI methods have 
good initial performance but saturate. We make this clearer in a later experiment. Overall, we see that
for networks with a relatively small number of parameters all EKI methods are comparable to SGD.
However with a large number of parameters, randomization is needed. This 
effect is particularly dominant 
when the ensemble size is relatively small; as we later show, larger ensemble sizes 
can perform significantly better.

\begin{figure}[t]
\centering 
    \begin{tabular}{ | l | c | c | c | c | }
        \hline
         & DNN 1 & DNN 2 & DNN 3 & DNN 4 \\ \hline
        SGD & 0.9199 & 0.9735 & 0.9798 & 0.9818 \\ \hline
        MSGD & 0.9257 & \textbf{0.9807} & \textbf{0.9830} & \textbf{0.9840} \\ \hline
        EKI & \vtop{\hbox{\strut \(\bar{u}\) \hspace{11px} 0.9092} \hbox{\strut \(u^{(j^*)}\) 0.9114}} &
        \vtop{\hbox{\strut \(\bar{u}\) \hspace{11px} 0.9398} \hbox{\strut \(u^{(j^*)}\) 0.9416}} &
        \vtop{\hbox{\strut \(\bar{u}\) \hspace{11px} 0.9424} \hbox{\strut \(u^{(j^*)}\) 0.9432}} &
        \vtop{\hbox{\strut \(\bar{u}\) \hspace{11px} 0.9404} \hbox{\strut \(u^{(j^*)}\) 0.9418}} \\ \hline
        MEKI & \vtop{\hbox{\strut \(\bar{u}\) \hspace{11px} 0.9094} \hbox{\strut \(u^{(j^*)}\) 0.9107}}  & \vtop{\hbox{\strut \(\bar{u}\) \hspace{11px} 0.9320} \hbox{\strut \(u^{(j^*)}\) 0.9332}} & n/a & n/a \\ \hline
        EKI(R) & \vtop{\hbox{\strut \(\bar{u}\) \hspace{11px} 0.9252} \hbox{\strut \(u^{(j^*)}\) \textbf{0.9260}}} &
        \vtop{\hbox{\strut \(\bar{u}\) \hspace{11px} 0.9721} \hbox{\strut \(u^{(j^*)}\) 0.9695}} &
        \vtop{\hbox{\strut \(\bar{u}\) \hspace{11px} 0.9738} \hbox{\strut \(u^{(j^*)}\) 0.9716}} &
        \vtop{\hbox{\strut \(\bar{u}\) \hspace{11px} 0.9741} \hbox{\strut \(u^{(j^*)}\) 0.9691}} \\ \hline
        MEKI(R) & \vtop{\hbox{\strut \(\bar{u}\) \hspace{11px} 0.9142} \hbox{\strut \(u^{(j^*)}\) 0.9162}} &
        \vtop{\hbox{\strut \(\bar{u}\) \hspace{11px} 0.9509} \hbox{\strut \(u^{(j^*)}\) 0.9511}} & n/a & n/a \\ \hline
     \end{tabular}
     \caption{Final test accuracies of six training methods on four dense neural networks, solving the MNIST 
     classification problem. Each bold number is the maximum across the column. For each EKI method we report the accuracy of the mean particle \(\bar{u}\)
     and of the best performing particle in the ensemble \(u^{(j^*)}\).}
     \label{DenseAccuracies}
\end{figure}

\begin{figure}[t]
    \centering
    \begin{subfigure}[b]{0.48\textwidth}
        \includegraphics[width=\textwidth]{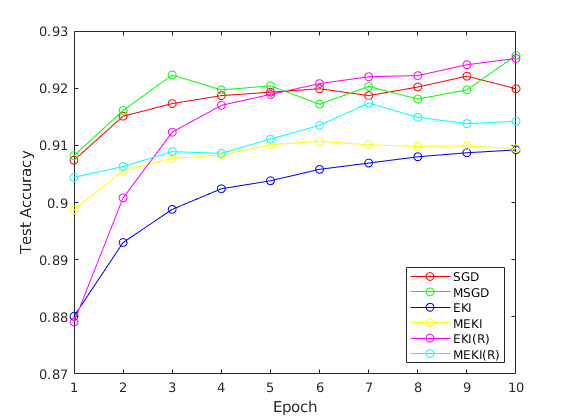}
        \caption{DNN 1}
    \end{subfigure}
    ~
    \begin{subfigure}[b]{0.48\textwidth}
        \includegraphics[width=\textwidth]{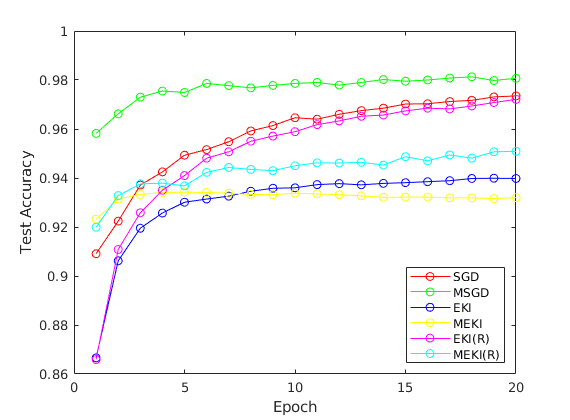}
        \caption{DNN 2}
    \end{subfigure}

    \begin{subfigure}[b]{0.48\textwidth}
        \includegraphics[width=\textwidth]{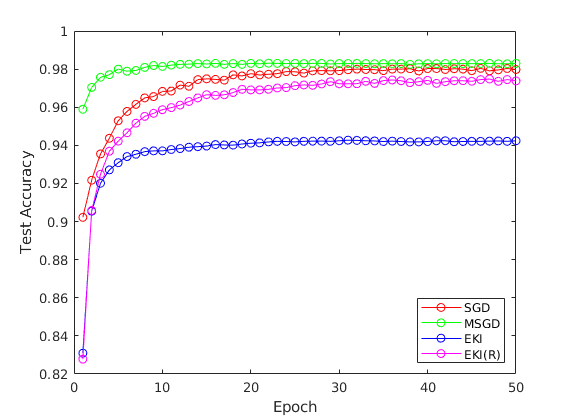}
        \caption{DNN 3}
    \end{subfigure}
    ~
    \begin{subfigure}[b]{0.48\textwidth}
        \includegraphics[width=\textwidth]{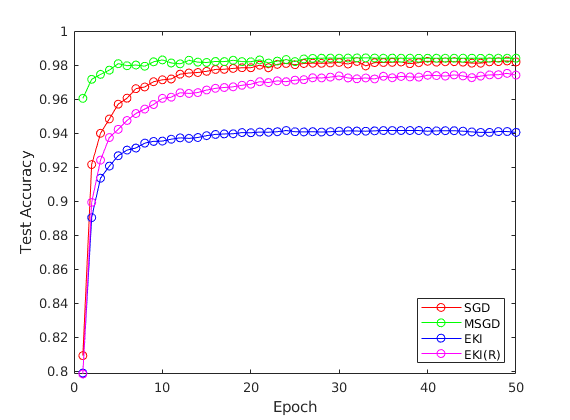}
        \caption{DNN 4}
    \end{subfigure}

    \caption{Test accuracies per epoch of six training methods on four dense neural networks, solving the MNIST 
     classification problem. For each EKI method the accuracy of the mean particle \(\bar{u}\) is shown.}
    \label{DenseGrowth}
\end{figure}

Figure \ref{DenseParticles} shows the test accuracies for each of the 
particles when using EKI on DNN-(1,2). We see 
that, the mean particle achieves roughly the average of the spread, as previously discussed. Our choice to use it as the
final parameter estimate is simply for convenience. One may use all the 
particles in a carefully weighted scheme
as an ensemble of networks and possibly achieve better results. Having many parameter estimates may also be 
advantageous when trying to avoid adversarial examples \cite{adversarial}. We leave these considerations to future work.

\begin{figure}[t]
    \centering
    \begin{subfigure}[b]{0.48\textwidth}
        \includegraphics[width=\textwidth]{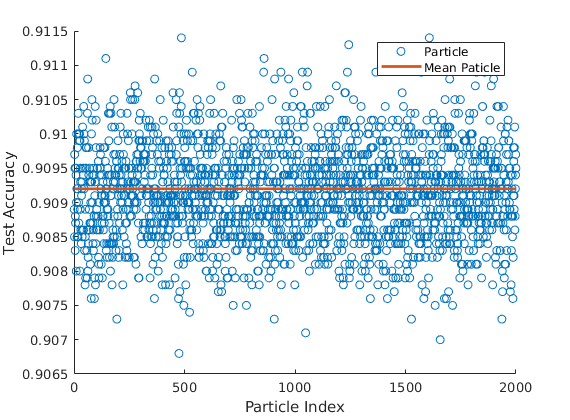}
        \caption{DNN 1}
    \end{subfigure}
    ~
    \begin{subfigure}[b]{0.48\textwidth}
        \includegraphics[width=\textwidth]{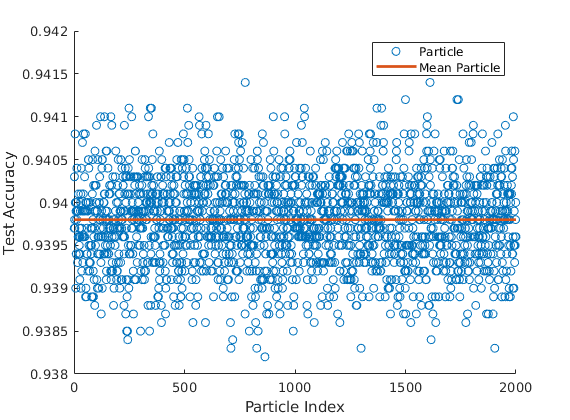}
        \caption{DNN 2}
    \end{subfigure}

    \caption{Particle accuracies of EKI on DNN-(1,2) compared to the accuracy of the mean particle \(\bar{u}\).}
    \label{DenseParticles}
\end{figure}

\begin{figure}[t]
    \centering
    \begin{subfigure}[b]{0.48\textwidth}
        \includegraphics[width=\textwidth]{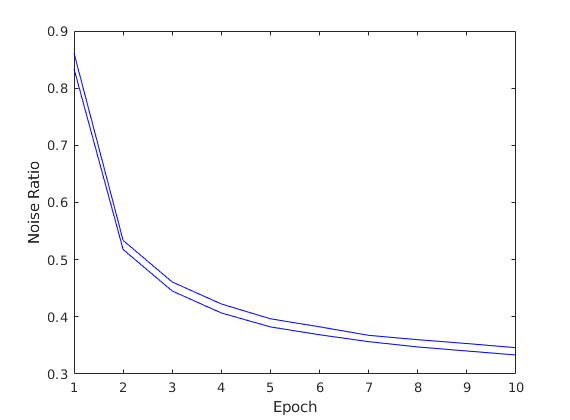}
        \caption{DNN 1}
    \end{subfigure}
    ~
    \begin{subfigure}[b]{0.48\textwidth}
        \includegraphics[width=\textwidth]{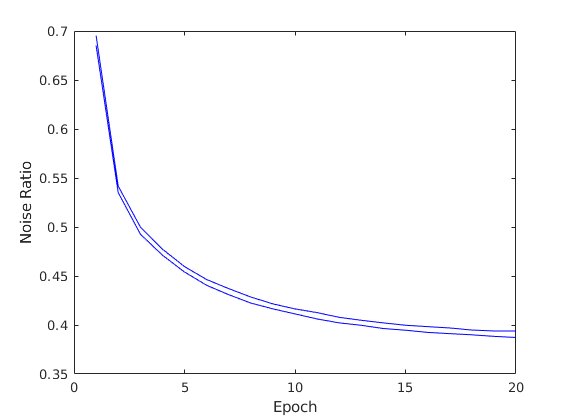}
        \caption{DNN 2}
    \end{subfigure}

    \caption{Spread of the noise ratio for EKI(R) on DNN-(1,2). At the end of every epoch, when the noise is added,
    the upper bound is computed as \(\|\bar{u}\|_2 / \max_j \|\eta^{(j)}\|_2\). The lower bound is computed analogously.}
    \label{NoiseSpread}
\end{figure}

To better illustrate the effect of the ensemble size, we compare all EKI methods on DNN 2
with an ensemble size of \(J = 6000\). The accuracies are shown in Figure \ref{DenseLarge}.
We again observe that the momentum methods perform very well initially, but fall off with
more training. This effect could be related to the specific 
time discretization method we use, but needs to be studied further theoretically and we 
leave this for future work. Note that with a larger ensemble, EKI is now comparable to SGD pointing out that
remaining in the subspace spanned by the initial ensemble is
a bottle neck for this method. On the other hand,
when we randomize, the ensemble size is no longer so relevant. EKI(R) performs
almost identically with 2,000 and with 6,000 ensemble members. Finding it to be the best method 
for these tasks, all experiments hereafter, unless stated otherwise, use EKI(R).

\begin{figure}[t]
\begin{minipage}[t]{0.48\textwidth}
\vspace{0pt}
\centering
\includegraphics[width=\textwidth]{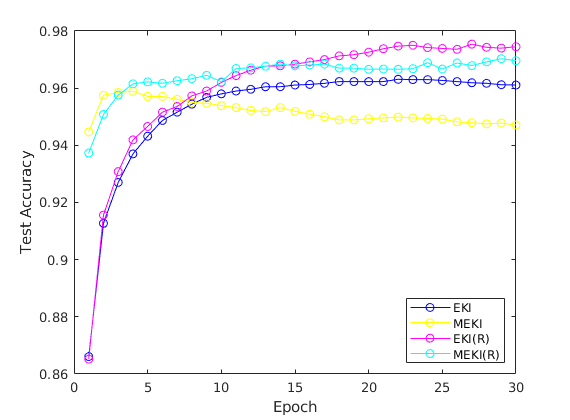}
\end{minipage}%
\begin{minipage}[t]{0.48\textwidth}
\vspace{50pt}
\centering
\begin{tabular}{ | l | c | c | }
        \hline
         & First & Final \\ \hline
        EKI & 0.8661 & 0.9611 \\ \hline
        MEKI & \textbf{0.9447} & 0.9471 \\ \hline
        EKI(R) & 0.8652 & \textbf{0.9745} \\ \hline
        MEKI(R) & 0.9373 & 0.9696  \\ \hline
\end{tabular}
\end{minipage}
\caption{Comparison of the test accuracies of four EKI methods on DNN 2 with ensemble size \(J=6000\).}
\label{DenseLarge}
\end{figure}

\subsubsection{Convolutional Neural Networks}

For our experiments with CNN(s), we employ both MNIST and SVHN. Since MNIST is a 
fairly easy data set, we can use a simple architecture and still achieve 
almost perfect accuracy. We name the model CNN-MNIST and its specifics are given in 
the first column of Figure \ref{ConvArchitectures}. SGD uses a learning rate 
of 0.05 while momentum SGD uses 0.01 and a momentum factor of 0.9. EKI(R)
has a fixed ensemble size of \(J=2000\). Figure \ref{ConvResults}
shows the results of training. We note that since CNN-MNIST uses ReLU and no
batch normalization, SGD struggles to find a good descent direction in the 
first few epochs. EKI(R), on the other hand, does not have this issue and 
exhibits a smooth test accuracy curve that is consistent with all other experiments.
In only 30 epochs, we are able to achieve almost perfect classification with 
EKI(R) slightly outperforming the SGD-based methods. 

Recent work suggests that 
the effectiveness of batch normalization does not come from dealing with the 
internal covariate shift, but, in fact, comes from smoothing the loss surface \cite{smoothingbachnorm}.
The noisy gradient estimates in EKI can be interpreted as doing the same thing and
is perhaps the reason we see smoother test accuracy curves. The contemporaneous
work of Haber et al \cite{haber} further supports this point of view.

\begin{figure}[t]
\centering
\begin{tabular}{ |c|c|c|c| }
\hline
\multicolumn{4}{ |c| }{Convolutional Neural Networks} \\
\hline
CNN-MNIST & CNN-1 & CNN-2 & CNN-3 \\ \hline
Conv 16x3x3 & Conv 16x3x3 & Conv 16x3x3 & Conv 16x3x3 \\ 
Conv 16x3x3 & Conv 16x3x3 & Conv 16x3x3 & Conv 16x3x3 \\
 & & Conv 16x3x3 & Conv 32x3x3 \\
 & & & Conv 32x3x3 \\
 & & & Conv 32x3x3 \\
 MaxPool 4x4 \((s=2)\) & MaxPool 2x2 & MaxPool 2x2 & MaxPool 2x2 \\ \hline
 Conv 16x3x3 & Conv 16x3x3 & Conv 32x3x3 & Conv 48x3x3 \\
 Conv 16x3x3 & Conv 16x3x3 & Conv 32x3x3 & Conv 48x3x3 \\
 & & Conv 32x3x3 & Conv 48x3x3 \\
 & & & Conv 48x3x3 \\
 & & & Conv 48x3x3 \\
 MaxPool 4x4 \((s=2)\) & MaxPool 2x2 & MaxPool 2x2 & MaxPool 2x2 \\ \hline
 Conv 12x3x3 & Conv 32x3x3 & Conv 48x3x3 & Conv 64x3x3 \\
 Conv 12x3x3 & Conv 32x3x3 & Conv 48x3x3 & Conv 64x3x3 \\
 & & Conv 64x3x3 & Conv 96x3x3 \\
 & & & Conv 96x3x3 \\
 & & & Conv 96x3x3 \\
 MaxPool 2x2 & MaxPool 8x8 & MaxPool 8x8 & MaxPool 8x8 \\ \hline
 FC-10 & FC-500 & FC-500 & FC-500 \\ 
 & FC-10 & FC-10 & FC-10 \\ \hline 
\end{tabular}
\caption{Architectures of 4 Convolutional Neural Networks with 6, 7, 10, 16 layers respectively from left to right. All convolutions
use a padding of 1, making them dimension preserving since all kernel sizes are 3x3. CNN-MNIST is evaluated on the
MNIST dataset and uses the ReLU nonlinearity. CNN-(1,2,3) are evaluated on the SVHN dataset and use the SELU nonlinearity. The convention \(s=2\)
refers to the stride of the max-pooling operation namely \(\alpha = \beta = 2\). All networks use a softmax thresholding.}
\label{ConvArchitectures}
\end{figure}

\begin{figure}[t]
    \centering
    \begin{subfigure}[b]{0.48\textwidth}
        \includegraphics[width=\textwidth]{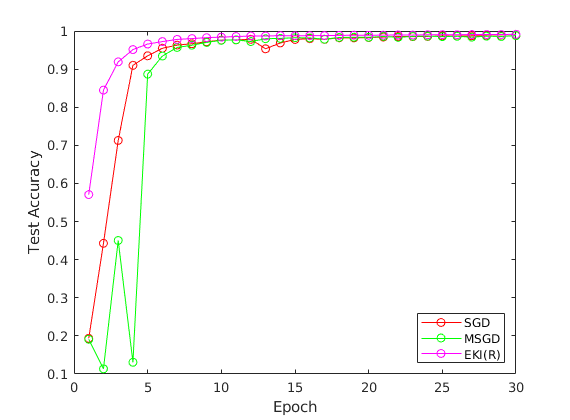}
        \caption{CNN-MNIST}
    \end{subfigure}
    ~ 
    \begin{subfigure}[b]{0.48\textwidth}
        \includegraphics[width=\textwidth]{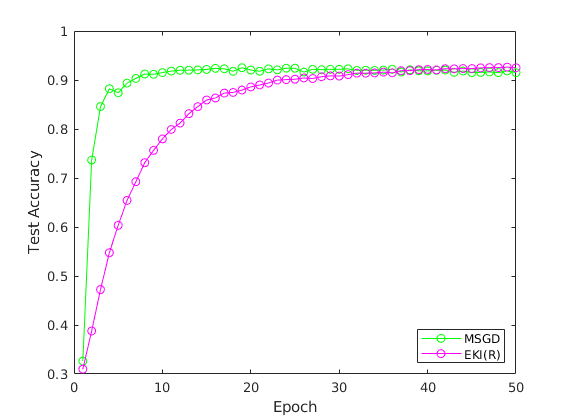}
        \caption{CNN-1}
    \end{subfigure}

    \begin{subfigure}[b]{0.48\textwidth}
        \includegraphics[width=\textwidth]{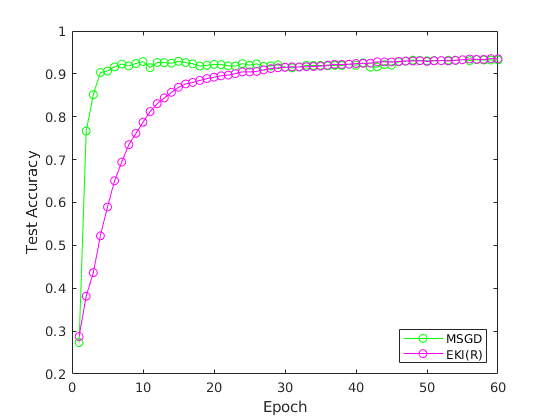}
        \caption{CNN-2}
    \end{subfigure}
    \begin{subfigure}[b]{0.48\textwidth}
        \includegraphics[width=\textwidth]{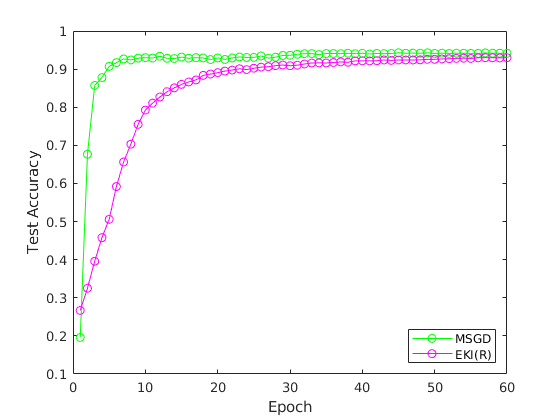}
        \caption{CNN-3}
    \end{subfigure}

    \begin{subfigure}[b]{0.9\textwidth}
        \vspace{10pt}
        \hspace{-30pt}
        \begin{tabular}{ |l|c|c|c|c|c|c|c|c| }
        \hline
        & \multicolumn{2}{ c| }{CNN-MNIST} & \multicolumn{2}{ c| }{CNN-1} & \multicolumn{2}{ c| }{CNN-2} & \multicolumn{2}{ c| }{CNN-3} \\
        \hline
        & First & Final & First & Final & First & Final & First & Final \\ \hline
        SGD & 0.1936 & 0.9878 & n/a & n/a & n/a & n/a & n/a & n/a \\ \hline
        MSGD & 0.1911 & 0.9880 & \textbf{0.3263} & 0.9150 & 0.2734 & 0.9324 & 0.1959 & \textbf{0.9414} \\ \hline
        EKI(R) & \textbf{0.5708} & \textbf{0.9912} & 0.3100 & \textbf{0.9249} & \textbf{0.2874} & \textbf{0.9353} & \textbf{0.2668} & 0.9299 \\ \hline
        \end{tabular}
    \end{subfigure}
    \caption{Comparison of the test accuracies of SGD and EKI(R) on four convolutional neural networks. SGD(M) refers to momentum SGD. CNN-MNIST is trained on the MNIST data set,
    while CNN-(1,2,3) are trained on the SVHN data set. }
    \label{ConvResults}
\end{figure}

Next we experiment on the SVHN data set with three CNN(s) of increasing
complexity. The architectures we use are inspired by those in \cite{goodinit}, and are referred to 
as Fit-Nets because each layer is shallow (has a relatively small number of parameters),
but the whole architecture is deep, reaching up to sixteen layers.
The details for the models dubbed CNN-(1,2,3) are given in Figure \ref{ConvArchitectures}.
Such models are known to be difficult to train; for this reason, the
papers \cite{fitnets,goodinit} present special
initialization strategies to deal with the model complexities. We find that 
when using the SELU nonlinearity and no batch normalization, simple Xavier initialization
works just as well. The results of training are presented in Figure \ref{ConvResults}.
We benchmark only against momentum SGD as all previous experiments show it performs better
than vanilla SGD. The method uses a learning rate of 0.01 and a momentum factor of 0.9. 
EKI(R) starts with \(J=200\) ensemble members and expands by 200 at end the of
each epoch until reaching a final ensemble of \(J=5000\). For CNN-3, memory 
constraints allowed us to only expand up a final size of \(J=2800\). All methods
use a batch size of 500. We see that,
in all three cases, EKI(R) and momentum SGD perform almost identically with EKI(R) slightly 
outperforming on CNN-(1,2), but falling off on CNN-3. This is likely due to the fact 
that CNN-3 has a large number of parameters and we were not able to provide a 
large enough ensemble size. This issue can be dealt with via parallelization by
splitting the ensemble among the memory banks of separate processing units.
We leave this consideration to future work.

\subsubsection{Recurrent Neural Networks}

For the classification task using a recurrent neural network, we return to the MNIST data set. Since
recurrent networks work on time series data, we split each image along its rows, making a 28-dimensional
time sequence with 28 entries, considering time going down from the top to the bottom of the image.
More complex strategies have been explored in \cite{rnnimage}.
We use a two-layer recurrent network with 32 hidden units and a \(\tanh\) nonlinearity, 
namely, in the notation of section \ref{recurrentneuralnetworks},
\[F_\theta(z,q) = \sigma \left ( W^{(2)}_h \sigma \left ( W^{(1)}_h z + b_h^{(1)} + W_x^{(1)}q + b_x^{(1)} \right ) + b_h^{(2)} + W_x^{(2)}q + b_x^{(2)} \right )\]
where \(\sigma = \tanh\) and \(W^{(1)}_h, W^{(2)}_h \in \mathbb{R}^{32 \times 32}\), \(W_x^{(1)}, W_x^{(2)} \in \mathbb{R}^{32 \times 28}\).
We look at only the last output of the network and thus only parametrize the last affine map \(A_{28}\). Softmax thresholding 
is applied. The initial hidden state is always taken to be 0.

We train with a batch size of 600 and SGD uses a learning rate of 0.05. EKI(R) starts with an
ensemble size of \(J=1000\) and expands by 1,000 at the end of every epoch until \(J=4000\) is
reached. Figure \ref{ReccurentMNIST} shows the result of training. EKI(R) performs significantly
better than SGD and appears more reliable, overall, for this task.

\begin{figure}[t!]
\begin{minipage}[t]{0.48\textwidth}
\vspace{0pt}
\centering
\includegraphics[width=\textwidth]{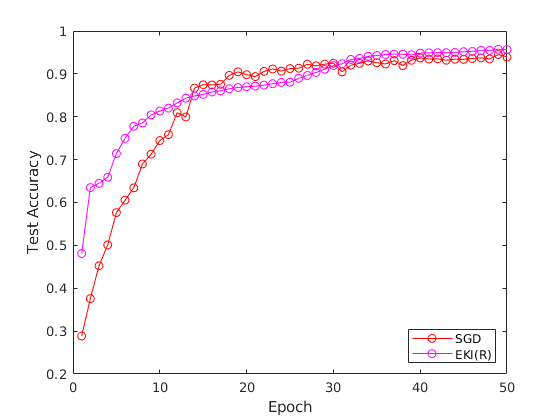}
\end{minipage}%
\begin{minipage}[t]{0.48\textwidth}
\vspace{60pt}
\centering
\begin{tabular}{ | l | c | c | }
        \hline
         & First & Final \\ \hline
        SGD & 0.2825 & 0.9391 \\ \hline
        EKI(R) & \textbf{0.4810} & \textbf{0.9566} \\ \hline
\end{tabular}
\end{minipage}
\caption{Comparison of the test accuracies of EKI(R) and SGD on the MNIST data set with a two layer recurrent neural network.}
\label{ReccurentMNIST}
\end{figure}

\subsection{Semi-supervised Learning}
\label{numericssemisupervised}

We proceed as in the construction of Example \ref{ex:2.2}, using the Voting Records data set.
For the affinity measure we pick \cite{perona,semisupervised}
\[\eta(x, y) = \text{exp} \left ( - \frac{\|x - y\|_2^2}{2  (1.25)^2} \right)\]
and construct the graph Laplacian \(\mathsf{L}(\mathsf{x})\). Its spectrum is shown in
Figure \ref{DataSamples}. Further we let \(\tau = 0\)
and \(\alpha = 1\), hence the prior covariance \(C = (\mathsf{L}(\mathsf{x}))^{-1}\)
is defined only on the subspace orthogonal to the first eigenvector of \(\mathsf{L}(\mathsf{x})\).
The most naive clustering algorithm that uses a graph Laplacian simply thresholds the
eigenvector of \(\mathsf{L}(\mathsf{x})\) that corresponds to the smalled non-zero eigenvalue
(called the Fiedler vector) \cite{spectraltutorial}. Its accuracy is shown in Figure \ref{SemiSuperResults}.
We found the best performing EKI method for this problem to simply be the vanilla
version of the method i.e. no randomization
or momentum. We use \(J = 1000\) ensemble members drawn from the prior and the 
mean squared-error loss. Its performance is only slightly better than the Fiedler vector 
as the particles quickly collapse to something close to the Fiedler vector. 
This is likely due to the fact that the initial ensemble is an i.i.d. sequence drawn from
the prior hence EKI converges to a solution in the subspace orthogonal to the 
first eigenvector of \(\mathsf{L}(\mathsf{x})\) which is close to the Fiedler vector, 
especially if the weights and other attendant hyper-parameters have been
chosen so that the Fielder vector already classifies the labeled nodes correctly. On the other hand,
the MCMC method detailed in \cite{semisupervised} can explore outside of this subspace and achieve 
much better results. We note, however, that EKI is significantly cheaper and 
faster than MCMC and thus could be applied to much larger problems where MCMC is not
computationally feasible.

\begin{figure}
\begin{minipage}[t]{0.48\textwidth}
\vspace{0pt}
\centering
\includegraphics[width=\textwidth]{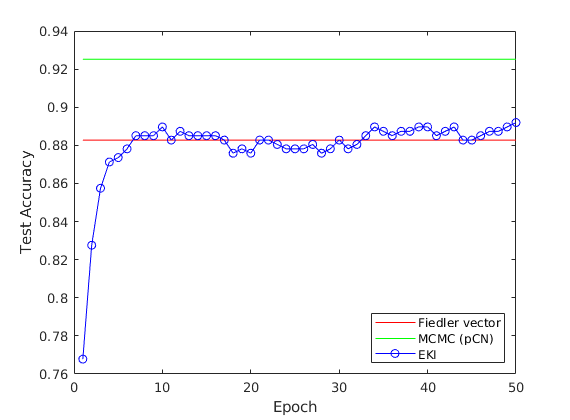}
\end{minipage}%
\begin{minipage}[t]{0.48\textwidth}
\vspace{60pt}
\centering
\begin{tabular}{ | l | c | }
        \hline
         & Accuracy \\ \hline
        Fiedler vector & 0.8828 \\ \hline
        MCMC (pCN) & \textbf{0.9252} \\ \hline
        EKI &  0.8920 \\ \hline
\end{tabular}
\end{minipage}
\caption{Comparison of the test accuracies of two semi-supervised learning algorithms to EKI on the Voting Records data set.}
\label{SemiSuperResults}
\end{figure}

\subsection{Online Learning}
\label{numericsonline}

Finally we consider two online learning problems using a recurrent neural network. 
We employ two univariate time-series data sets: minimum daily temperatures in 
Melbourne, and the monthly number of sunspots observed from Z{\"u}rich.
For both, we use a single layer recurrent network with 32 hidden units and the \(\tanh\) nonlinearity.
The output is not thresholded i.e. \(S = id\). At the initial time, we set the 
hidden state to 0 then use the hidden state computed in the previous step to initialize
for the current step. This is an online problem as our algorithm only sees one data-label
pair at a time. For OGD, we use a learning rate of 0.001 while, for
EKI, we use \(J=12\) ensemble members. Figure \ref{TimeSeriesResults} shows the results of training 
as well as how well each of the trained model fits the test data.
Notice that EKI converges much more quickly and to a slightly better solution 
than OGD in both cases. Furthermore, the model learned by EKI is able to
better capture small scale oscillations. These are very promising results 
for the application of EKI to harder RNN problems.

\begin{figure}[t]
    \centering
    \begin{subfigure}[b]{0.3\textwidth}
        \includegraphics[width=\textwidth]{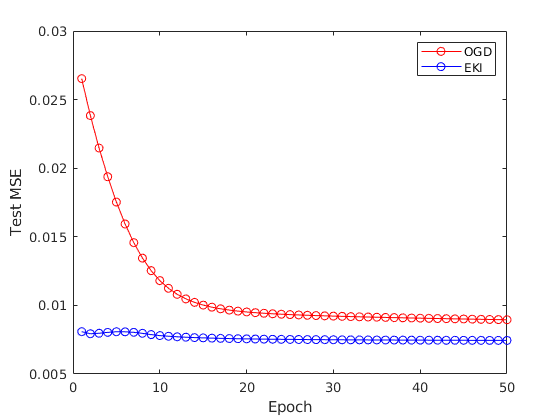}
    \end{subfigure}
    ~ 
    \begin{subfigure}[b]{0.3\textwidth}
        \includegraphics[width=\textwidth]{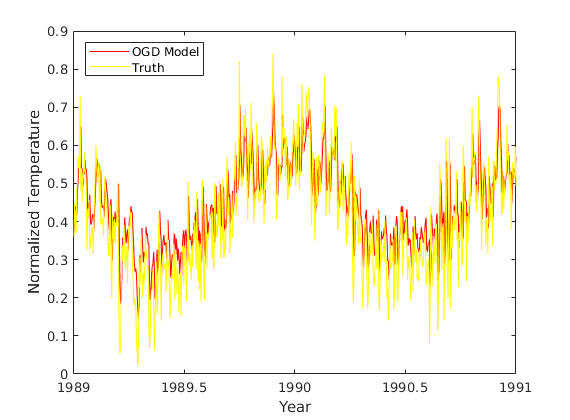}
    \end{subfigure}
     ~ 
    \begin{subfigure}[b]{0.3\textwidth}
        \includegraphics[width=\textwidth]{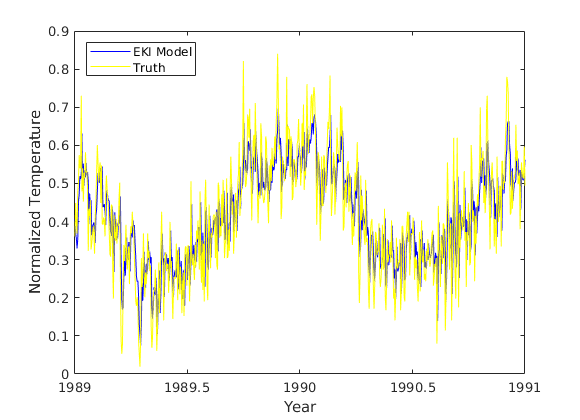}
    \end{subfigure}

    \begin{subfigure}[b]{0.3\textwidth}
        \includegraphics[width=\textwidth]{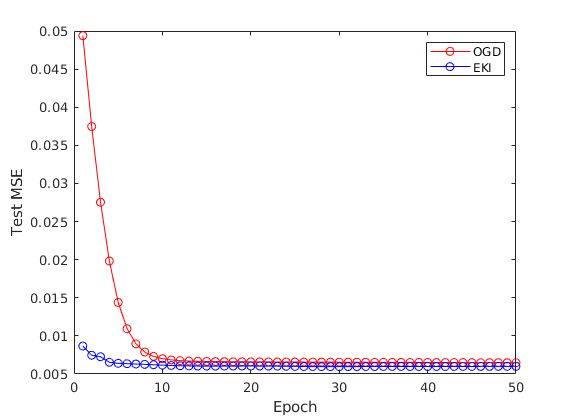}
    \end{subfigure}
    ~ 
    \begin{subfigure}[b]{0.3\textwidth}
        \includegraphics[width=\textwidth]{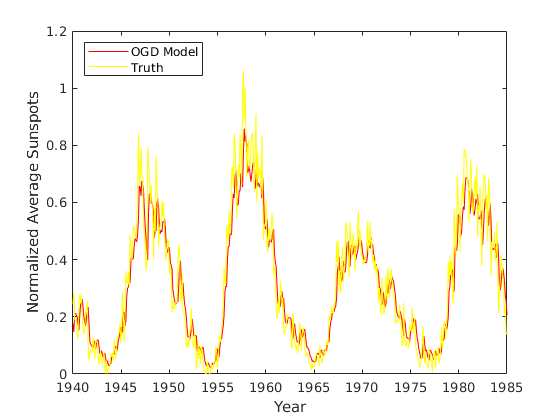}
    \end{subfigure}
     ~ 
    \begin{subfigure}[b]{0.3\textwidth}
        \includegraphics[width=\textwidth]{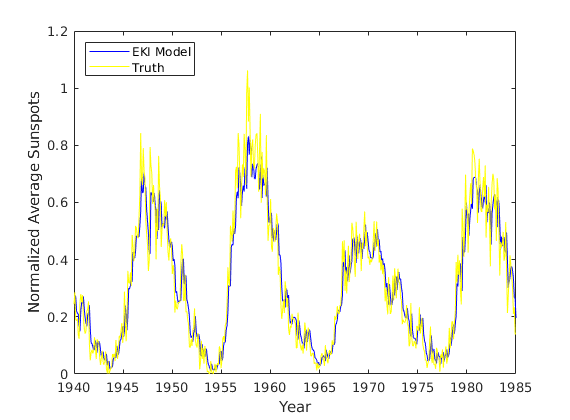}
    \end{subfigure}

    \begin{subfigure}[b]{0.83\textwidth}
        \vspace{10pt}
        \begin{tabular}{ |l|c|c|c|c| }
        \hline
        & \multicolumn{2}{ c| }{Melbourne Temperatures} & \multicolumn{2}{ c| }{Z{\"u}rich Sunspots}  \\
        \hline
        & First & Final & First & Final \\ \hline
        OGD & \(2.653 \times 10^{-2}\) & \(8.954 \times 10^{-3}\) & \(4.939 \times 10^{-2}\) & \(6.480 \times 10^{-3}\)  \\ \hline
        EKI & \(\mathbf{8.086 \times 10^{-3}}\) & \(\mathbf{7.448 \times 10^{-3}}\) & \(\mathbf{8.671 \times 10^{-3}}\) & \(\mathbf{6.006 \times 10^{-3}}\) \\ \hline
        \end{tabular}
    \end{subfigure}

    \caption{Comparison of OGD and EKI on two online learning tasks with a recurrent neural network.
    The top row shows the minimum daily temperatures in Melbourne data set, while the bottom
    shows the number of sunspots observed each month from Z{\"u}rich data set.} 
    \label{TimeSeriesResults}
\end{figure}

\section{Conclusion and Future Directions}
\label{conclusion}

We have demonstrated that many machine learning problems can easily fit into the unified framework 
of Bayesian inverse problems. Within this framework, we apply Ensemble Kalman Inversion methods, for which we 
suggest suitable modifications, to tackle such tasks. Our numerical experiments suggest a 
wide applicability and competitiveness against the state-of-the-art for our schemes. The following
directions for future search arise naturally from our work:

\begin{itemize}
    \item Theoretical analysis of the momentum and general loss EKI methods as
well as their possible application to physical inverse problems.
    \item GPU parallelization of EKI methods and its application to large scale machine learning tasks.
    \item Application of EKI methods to more difficult recurrent neural network problems as well as 
    problems in reinforcement learning.
    \item Use of the entire ensemble of particle estimates to improve accuracy and possibly combat
    adversarial examples.

\end{itemize}

\section{Acknowledgments}

Both authors are supported, in part, by the US National Science Foundation (NSF)
grant DMS 1818977, the US Office of Naval Research (ONR) grant N00014-17-1-2079, and
the US Army Research Office (ARO) grant W911NF-12-2-0022.

\section*{References}
\bibliographystyle{siam}
\bibliography{references}

\end{document}